\documentclass{article}

\usepackage{microtype}
\usepackage{graphicx}
\usepackage{xcolor} % for comment commands (can be removed at the end)
\usepackage{subfigure}
\usepackage{booktabs} % for professional tables

\usepackage[numbers]{natbib}

\usepackage{caption}

\usepackage{floatrow}
\newfloatcommand{capbtabbox}{table}[][\FBwidth]

\usepackage{algorithm}
\usepackage{algorithmic}

\usepackage{listings}

\usepackage{enumerate}
\usepackage{amsmath,amsthm,bm,amsfonts}
\usepackage[export]{adjustbox}

\usepackage{wrapfig}

\usepackage{booktabs}
\usepackage{floatrow}
\newtheorem{theorem}{Theorem}

\newcommand{\bx}{\mathbf{x}}
\newcommand{\by}{\mathbf{y}}
\newcommand{\bK}{\mathbf{G}}
\newcommand{\bk}{\mathbf{g}}

 \usepackage[preprint]{nips_2018}

\usepackage[utf8]{inputenc} % allow utf-8 input
\usepackage[T1]{fontenc}    % use 8-bit T1 fonts
\usepackage{hyperref}       % hyperlinks
\usepackage{url}            % simple URL typesetting
\usepackage{booktabs}       % professional-quality tables
\usepackage{amsfonts}       % blackboard math symbols
\usepackage{nicefrac}       % compact symbols for 1/2, etc.
\usepackage{microtype}      % microtypography

\title{Active Mini-Batch Sampling \\ using Repulsive Point Processes}

\author{
  Cheng Zhangs\thanks{This work was initialized when the author was at Disney Research, and was carried out when the author was at KTH and Microsoft Research.} \\
  Microsoft Research\\
  Cambridge, UK \\
 \texttt{Cheng.Zhang@microsoft.com} \\
 \And
 Cengiz \"Oztireli \\
 Disney Research\\
 Zurich, Switzerland\\
 \texttt{cengiz.oztireli@disneyresearch.com}\\
 \And
 Stephan Mandt \\
 Disney Research\\
 Los Angeles, USA \\
 \texttt{stephan.mandt@disneyresearch.com}\\
 \And
 Giampiero Salvi\\
 KTH Royal Institute of Technology\\
 Stockholm, Sweden\\
 \texttt{giampi@kth.se}
}

\begin{document}
\maketitle
%\vspace{-15pt}
\begin{abstract} 
The convergence speed of stochastic gradient descent (SGD) can be improved by actively selecting mini-batches. 
We explore sampling schemes where similar data points are less likely to be selected in the same mini-batch. 
In particular, we prove that such repulsive sampling schemes lower the variance of the gradient estimator. This generalizes recent work  on using Determinantal Point Processes (DPPs) for mini-batch diversification (Zhang et al., 2017) to the broader class of repulsive point processes.
We first show that the phenomenon of variance reduction by diversified sampling generalizes in particular to
non-stationary point processes. 
We then show that other point processes may be computationally much more efficient than DPPs. In particular, we propose and investigate Poisson Disk sampling---frequently encountered in the computer graphics community---for this task.
We show empirically that our approach improves over standard SGD both in terms of convergence speed as well as final model performance. 
\end{abstract} 
%\vspace{-15pt}
\section{Introduction}
\label{sec:intro}

Stochastic gradient descent (SGD) \cite{bottou2010large} is key to modern scalable machine learning. Combined with back-propagation, it forms the foundation to train deep neural networks \cite{lecun1998efficient}. Applied to variational inference \cite{hoffman13stochastic,zhang2017advances}, it enables the use of probabilistic graphical models on massive data. SGD training has contributed to breakthroughs in many applications \citep{krizhevsky2012imagenet,mikolov2013efficient}.

A key limitation for the speed of convergence of SGD algorithms is the stochastic gradient noise. Smaller gradient noise allows for larger learning rates and therefore faster convergence.
For this reason, variance reduction for SGD is an active research topic.

Several recent works have shown that the variance of SGD can be reduced when diversifying the data points in a mini-batch based on their features
\cite{zhao2014accelerating, fu2017CPSGMCMC, Zhang17Stochastic}. 
When data points are coming from similar regions in feature space, their gradient contributions are positively correlated. 
Diversifying data points by sampling from different regions in feature space de-correlates their gradient contributions, which leads to a better gradient estimation.

Another benefit of actively biasing the mini-batch sampling procedure relates to better model performance
\cite{chang2017active,Shrivastava_2016_CVPR,Zhang17Stochastic}.  
Zhang et al. \cite{Zhang17Stochastic} biased the data towards a more uniform distribution, upsampling data-points in scarce regions and downsampling data points in dense regions, leading to a better performance during test time. 
Chen et al. \cite{chen2015webly} showed that training on simple classification tasks first, and later adding more difficult examples, leads to a clear performance gain compared to training on all examples simultaneously. We refer to such schemes which modify the marginal probabilities of each selected data point as \emph{active bias}. The above results suggest that utilizing an active bias in mini-batch sampling can result in improved performance without additional computational cost.

In this work, we present a framework for active mini-batch sampling based on repulsive point processes. The idea is simple: we specify a data selection mechanism that actively selects a mini-batch based on features of the data. This mechanism introduces \emph{repulsion} between the data points, meaning that it suppresses data points with similar features to co-occur in the same mini-batch. We use repulsive point processes for this task. Finally, the chosen mini-batch is used to perform a stochastic gradient step, and this scheme is repeated until convergence.

Our framework generalizes the recently proposed mini-batch diversification based on determinantal point processes (DPP) \cite{Zhang17Stochastic} to a much broader class of repulsive processes, and allows users to encode any preferred active bias with efficient mini-batch sampling algorithms.
In more detail, our contributions are as follows:
\begin{enumerate}
\item  \textit{We propose to use point processes for active mini-batch selection (Section~\ref{sec:theory}). } \\
   We provide a theoretical analysis which shows that mini-batch selection with repulsive point processes may reduce the variance of stochastic gradient descent. 
   The proposed approach can accommodate point processes with adaptive densities and adaptive pair-wise correlations. Thus, we can use it for data subsampling with an active bias.
 %    \vspace{-4pt}

\item \textit{Going beyond DPPs, we propose a group of more efficient repulsive point processes based on Poisson Disk Sampling (PDS) (Section \ref{sec:method}).}\\
   We propose PDS with dart throwing for mini-batch selection. Compared to DPPs, this improves the sampling costs from $Nk^3$ to merely $k^2$, where $N$
   is the number of data points 
   and $k$ is the mini-batch size (Section \ref{sec:VannilaPD}). \\ 
   We propose a dart-throwing method with an adaptive disk size and adaptive densities 
   to sample mini-batches with an active bias
   (Section \ref{sec:GenerlizedPD}).  
\item \textit{
We test our proposed method on several machine learning applications (Section \ref{sec:exp}) from the domains of computer vision and speech recognition. We find increased model performance and faster convergence due to variance reduction. }
\end{enumerate}
%} 
%\CZ{I still prefer the bullet point style personally.}
%\vspace{-4pt}
\section{Related Work}
\label{sec:related}
%\vspace{-3pt}
In this section, we begin with reviewing the most relevant work on \textit{diversified mini-batch sampling}. Then, we discuss the benefits of subsampling schemes with an \textit{active bias}, where data are either reweighed or re-ordered. Finally, we review the relevant aspects of \textit{point processes}.

\paragraph{Diversified Mini-Batch Sampling} Prior research \cite{zhao2014accelerating,fu2017CPSGMCMC, Zhang17Stochastic, yin2017gradient} has shown that sampling diversified mini-batches can reduce the variance of stochastic gradients. It is also the key to overcome the problem of the saturation of the convergence speed in the distributed setting \cite{yin2017gradient}. Diversifying the data is also computationally efficient for large-scale learning problems 
\cite{zhao2014accelerating,fu2017CPSGMCMC, Zhang17Stochastic}. 

 Zhang et. al.~\cite{Zhang17Stochastic} recently proposed to use DPPs for diversified mini-batch sampling and drew the connection to stratified sampling \cite{zhao2014accelerating}  
and clustering-based preprocessing for SGD \cite{fu2017CPSGMCMC}.
A disadvantage of the DPP-approach is the computational overhead.
Besides presenting a more general theory, we provide more efficient point processes in this work. 

\paragraph{Active Bias} Different types of active bias in subsampling the data can improve the convergence and lead to improved final performance in model training \cite{alain2015variance,gopal2016adaptive, chang2017active,chen2015webly, Shrivastava_2016_CVPR}. As summarized in \cite{chang2017active}, self-paced learning biases towards easy examples in the early learning phase. Active-learning, on the other hand, puts more emphasis on uncertain cases, and hard example mining focuses on the difficult-to-classify examples. 

Chang et. al. ~\cite{chang2017active} investigate a supervised setup and sample data points, which have high prediction variance, more often. \cite{gopal2016adaptive}
maintain a desired class distribution during mini-batch sampling. Diversified mini-batch sampling with DPPs \cite{Zhang17Stochastic} re-weights the data towards a more uniform distribution, which improves the final performance when the data set is imbalanced.

The choice of a good active bias depends on the data set and problem at hand. Our proposed method is compatible with different active bias preferences.

\paragraph{Point Processes} Point processes have a long history in physics and mathematics, and are heavily used in computer graphics \cite{macchi1975coincidence,ripley1976second,illian2008statistical,Oztireli12Analysis,lavancier2015determinantal}. DPPs, as a group of point processes, have been introduced and used in the machine learning community in recent years \cite{kulesza2012determinantal,li2015efficient,kathuria2016batched}. 

Other types of point processes have been less explored and used in machine learning. There are many different repulsive point processes, such as 
PDS, or Gibbs processes, with properties similar to DPPs, but with significantly higher sampling efficiency  \cite{illian2008statistical}. Additionally, more flexible point processes with adaptive densities and interactions are well studied in computer graphics \cite{li2010anisotropic,Rov17a,kita2016multi}, but not explored much in the machine learning community. Our proposed framework is based on generic point processes. As one of the most efficient repulsive point processes, we advocate Poisson disk sampling in this paper.

\section{Repulsive Point Processes for Variance Reduction}
\label{sec:theory}
In this section, we first briefly introduce our main idea of using point processes for mini-batch sampling in the context of the problem setting (Section \ref{sec:probelmS}) and revisit point processes (Section \ref{sec:point_processes}). 
We prove that any repulsive point process can lead to reduced gradient variance in SGD (Section \ref{sec:PP_minibatch}), and discuss the implications of this result. The theoretical analysis in this section leads to multiple practical algorithms 
in Section \ref{sec:method}. 

\subsection{Problem Setting}
\label{sec:probelmS}
\begin{figure}%[htb]
\centering
\subfigure[Non-Repulsive]{
\includegraphics[width=3.2cm,frame]{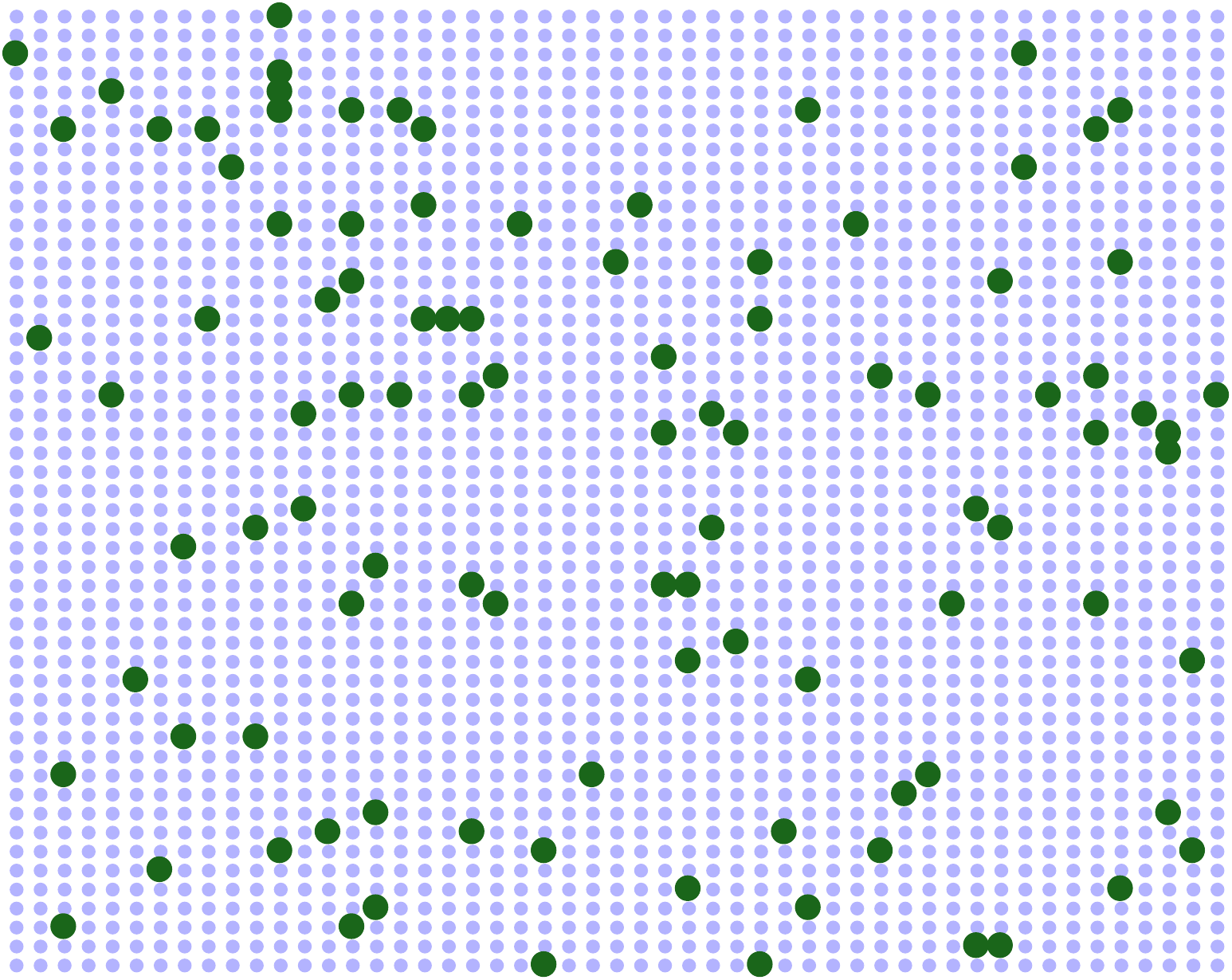}
\label{fig:toysubset:random}
}
\hspace{5mm}
\subfigure[Stationary]{
\includegraphics[width=3.2cm,frame]{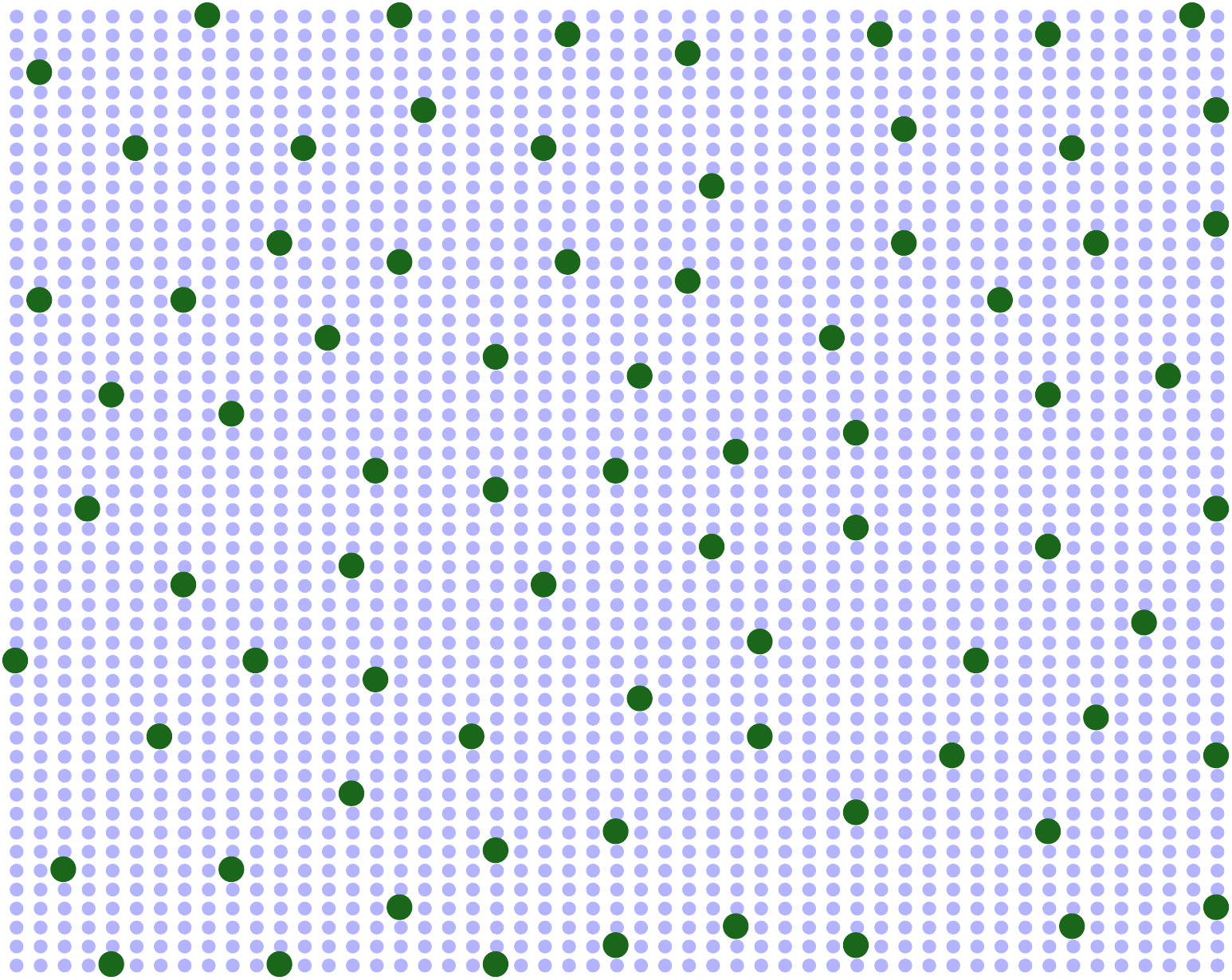}
\label{fig:toysubset:stationary}
}
\hspace{5mm}
\subfigure[Non-Stationary]{
\includegraphics[width=3.2cm,frame]{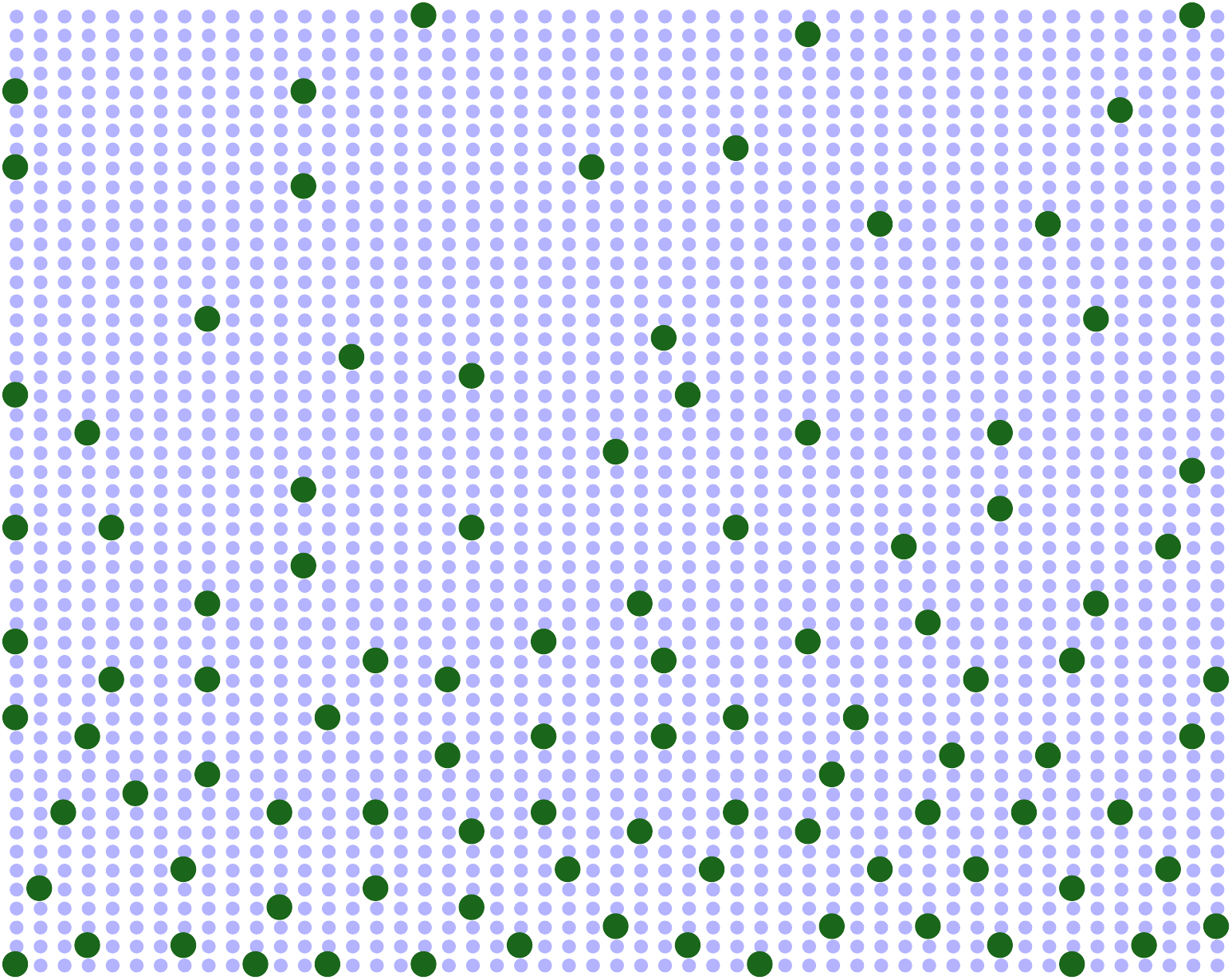}
\label{fig:toysubset:dense}
}
%\vspace{-5pt}
\caption{%\footnotesize
Examples of sampling a subset of data points.  We sampled different subsets of 100 points (dark green) from a bigger set of points (light blue), using three different point processes. Panel ~\ref{fig:toysubset:random} shows a uniformly randomly sampled subset. Panel~\ref{fig:toysubset:stationary} and panel ~\ref{fig:toysubset:dense} show two examples with different repulsive point processes.
}
\label{fig:toysubset}
\end{figure}
Consider a loss function $\ell(\bx,\theta)$, where $\theta$ are the model parameters, and $\bx$ indicates the data. 
In this paper, we consider a modified empirical risk minimization problem ~\citep{Zhang17Stochastic}:
\begin{equation}
\hat{J}(\theta) = \mathbb{E} _{\hat{\bx} \sim \mathcal{P}}[\ell(\hat{\bx}, \theta) ] .
\label{eq:machLearnGenFormEstimated}
\end{equation}
 $\mathcal{P}$ indicates a point process defining a distribution over subsets $\hat{\bx}$ of the data, which will be specified below. Note that this leads to a  potentially biased version of the standard empirical risk \citep{bottou2010large}.. 

We optimize Eq.~\ref{eq:machLearnGenFormEstimated} via stochastic gradient descent, which leads to the updates

\begin{equation}
\label{eq:biased-update}
\theta_{t+1} = \theta_t - \rho_t \hat{G}  
=\theta_t - \rho_t {\textstyle \frac{1}{|B|}}\sum_{i \in B} \nabla \ell(x_i,\theta), \quad B\sim {\rm \mathcal{P}}.
\end{equation}
$B \subset \{1,\dots,N\}$, a set of data indices that define the mini-batch.
$\hat{G}$ is the gradient estimated from a mini-batch. 
The data points chosen for each mini-batch are drawn from a point process $\mathcal{P}$, which defines probability measures over different mini-batches.
Therefore, our scheme generalizes SGD in that the data points in the mini-batch $B$ are selected actively, rather than uniformly.

Figure~\ref{fig:toysubset} shows examples of subset selection using different point processes.
Drawing the data randomly without replacement corresponds to a point process as well, thus standard SGD trivially belongs to the class of algorithms considered here. In this paper, we investigate different point processes and analyze how they improve the performance of different models on empirical data sets.

\subsection{Background on Point Processes}
\label{sec:point_processes}

Point processes are generative processes of collections of points in some measure space~\cite{Moller03Statistical,illian2008statistical}.
They can be used to sample subsets of data with various properties, either from continuous spaces or from discrete sets, such as a finite dataset.
In this paper, we explore different point processes
to sample mini-batches with different statistical properties.

More formally, a point process $\mathcal{P}$ in $\mathbb{R}^d$ can be defined by considering the joint probabilities of finding points generated by this process in infinitesimal volumes. One way to express these probabilities is via \emph{product densities}. Let $\bx_i$ denote some arbitrary points in $\mathbb{R}^d$, and $\mathcal{B}_i$ infinitesimal spheres centered at these points with volumes $d\bx_i = | \mathcal{B}_i |$. Then the $n^{th}$ order product density $\varrho^{(n)}$ is defined by 
$$
p(\bx_1, \cdots, \bx_n) = \varrho^{(n)} (\bx_1, \cdots, \bx_n) d\bx_1 \cdots d\bx_n,$$
where $p(\bx_1, \cdots, \bx_n)$ is the joint probability of having a point of the point process $\mathcal{P}$ in each of the infinitesimal spheres $\mathcal{B}_i$. We can use $\mathcal{P}$ to generate infinitely many point configurations, each corresponding to e.g. a mini-batch.

For example, DPP defines this probability of sampling a subset as being proportional to the determinant of a kernel matrix. 
 It is thus described by the $n^{\text{th}}$ order product density~\cite{lavancier2012statistical}: $
\varrho^{(n)}(\bx_1, \cdots, \bx_n) = det[C](\bx_1, \cdots, \bx_n)$,
where 
$det[C](\bx_1, \cdots, \bx_n)$ is the determinant of the $n \times n $ sized sub-matrix of kernel matrix C with entries specified by $\bx_1, \cdots, \bx_n$.

For our analysis, we will just need the first and second order product density, which are commonly denoted by $\lambda(\bx):= \varrho^{(1)}(\bx)$, $\varrho(\bx, \mathbf{y}) := \varrho^{(2)}(\bx, \mathbf{y})$. An important special case of point processes is stationary processes. For such processes, the point distributions generated are translation invariant, where the intensity is a constant.

\subsection{Point Processes for Active Mini-Batch Sampling}
\label{sec:PP_minibatch}
Recently, Zhang et al.\cite{Zhang17Stochastic} investigated how to utilize a particular type of point process, DPP, for mini-batch diversification.
Here, we generalize the theoretical results to arbitrary stochastic point processes, and elaborate on how the resulting formulations can be utilized for SGD based algorithms. This opens the door to exploiting a vast literature on the theory of point processes, and efficient algorithms for sampling.

SGD-based algorithms utilize the estimator $\hat{\mathbf{G}}(\theta) = \frac{1}{|B|}\sum_{i \in B} \nabla \ell(x_i,\theta)$ for the gradient of the objective   (Section~\ref{sec:probelmS}). Each mini-batch, i.e. set of data points in this estimator, can be considered as an instance of an underlying point process $\mathcal{P}$. 
Our goal is to design sampling algorithms for improved learning performance by altering the bias and variance of this gradient estimator.

We first derive a closed form formula for the variance $\mathrm{var}_{\mathcal{P}} (\hat{\bK})$ of the gradient estimator for general point processes. We then show that, under mild regularity assumptions, repulsive point processes generally imply variance reduction. For what follows, let $\mathbf{g}(\mathbf{x},\theta) = \nabla \ell(\bx,\theta)$ denote the gradient of the loss function, and recall that $k = |B|$, the mini-batch size.

\begin{theorem}
The variance $\mathrm{var}_{\mathcal{P}} (\hat{\bK})$ of the gradient estimate $\hat{\bK}$ in SGD for a general stochastic point process $\mathcal{P}$ is given by:
\begin{equation}
\small
\label{eq:var_formula}
%\vspace{-5pt}
\mathrm{var}_{\mathcal{P}} (\hat{\bK}) =  \frac{1}{k^2} \int_{\mathcal{V}\times \mathcal{V}}  \lambda(\bx)  \lambda(\by) \bk(\bx, \theta)^T  \bk(\by, \theta) \left[ \frac{\varrho(\bx, \by)}{\lambda(\bx) \lambda(\by)} - 1 \right] d\bx d\by +\frac{1}{k^2} \int_{\mathcal{V}}{ \Vert \bk(\bx, \theta) \Vert^2 \lambda(\bx) d\bx}. 
\vspace{-5pt}
\end{equation}

\end{theorem}
\begin{proof}
%\vspace{-2pt}
In Appendix~\ref{app:derivation_variances}
\end{proof}

\paragraph{Remark 1.} This formula applies to general point processes and hence sampling strategies for mini-batches. It proves that variance only depends on first and second order correlations captured by $\lambda(\bx)$ and $\varrho(\bx, \by)$, respectively. This provides a simple and convenient tool for analyzing properties of sampling strategies with respect to dataset characteristics for variance control, once only these lower order sampling characteristics are known or estimated by simulation.

\paragraph{Remark 2.} For standard SGD, we have $\varrho(\bx, \mathbf{y}) = \lambda(\bx) \lambda(\mathbf{y})$. This is due to the nature of random sampling, where sampling a point is independent of already sampled points. Note that this applies also to adaptive sampling with non-constant $\lambda(\bx)$. Hence, the term $[\frac{\varrho(\bx, \by)}{\lambda(\bx) \lambda(\by)} - 1 ]$ vanishes in SGD. In contrast, we show next that this term may induce a variance reduction for repulsive point processes.

\paragraph{Remark 3.} Repulsive point processes may make the first term in Eq.~\ref{eq:var_formula} negative, implying variance reduction. For repulsive point processes, the probability of sampling points that are close to each other is low. Consequently, if $\bx$ and $\by$ are close, $\varrho(\bx, \by) < \lambda(\bx)\lambda(\by)$, and the term $[\frac{\varrho(\bx, \by)}{\lambda(\bx) \lambda(\by)} - 1 ]$ is negative.  
This is due to points repelling each other (we will elaborate more on this in the next section). Furthermore, assuming that the loss function is sufficiently smooth in its data argument, the gradients are aligned for close points i.e $\bk(\bx, \theta)^T  \bk(\by, \theta) >0$. 
This combined implies that close points provide negative contributions to the first integral in Eq.~\ref{eq:var_formula}. The contributions of points farther apart average out and become negligible due to gradients not being correlated with $\varrho(\bx, \by)$, which is the case for all current sampling algorithms and the ones we propose in the next section. The negative first term in Eq.~\ref{eq:var_formula} leads to variance reduction, for repulsive point processes.

\paragraph{Implications.}
\label{sec:benifit}
%\vspace{-2pt}
This proposed theory allows us to use any point process for mini-batch sampling, such as DPP, finite Gibbs processes, Poisson disk sampling (PDS), and many more \cite{illian2008statistical}.
It thus offers many new directions of possible improvement. 
Foremost, we can choose point processes with a different degree of repulsion \cite{biscio2016quantifying}, and computational efficiency. 
Furthermore, in this general theory, we are able to adapt the density 
and alter the pair-wise interactions to encode our preference.
In the next section, we propose several practical algorithms utilizing these benefits.

\section{Poisson Disk Sampling for Active Mini-batch Sampling}
\label{sec:method}

We adapt efficient dart throwing algorithms for fast repulsive and adaptive mini-batch sampling (Section \ref{sec:VannilaPD}).
 
We further 
extend the algorithm with an adaptive disk size and 
density (Section \ref{sec:GenerlizedPD}). For supervised setups, we shrink the disc size towards the decision boundary, using mingling indices \cite{illian2008statistical}. This biases towards hard examples and improves classification accuracy. 

\subsection{Stationary Poisson Disk Sampling}
\label{sec:VannilaPD}

\begin{wrapfigure}{l}{0.4\textwidth}
%\vspace{-15pt}
  \centering
  \includegraphics[width=3.9cm]{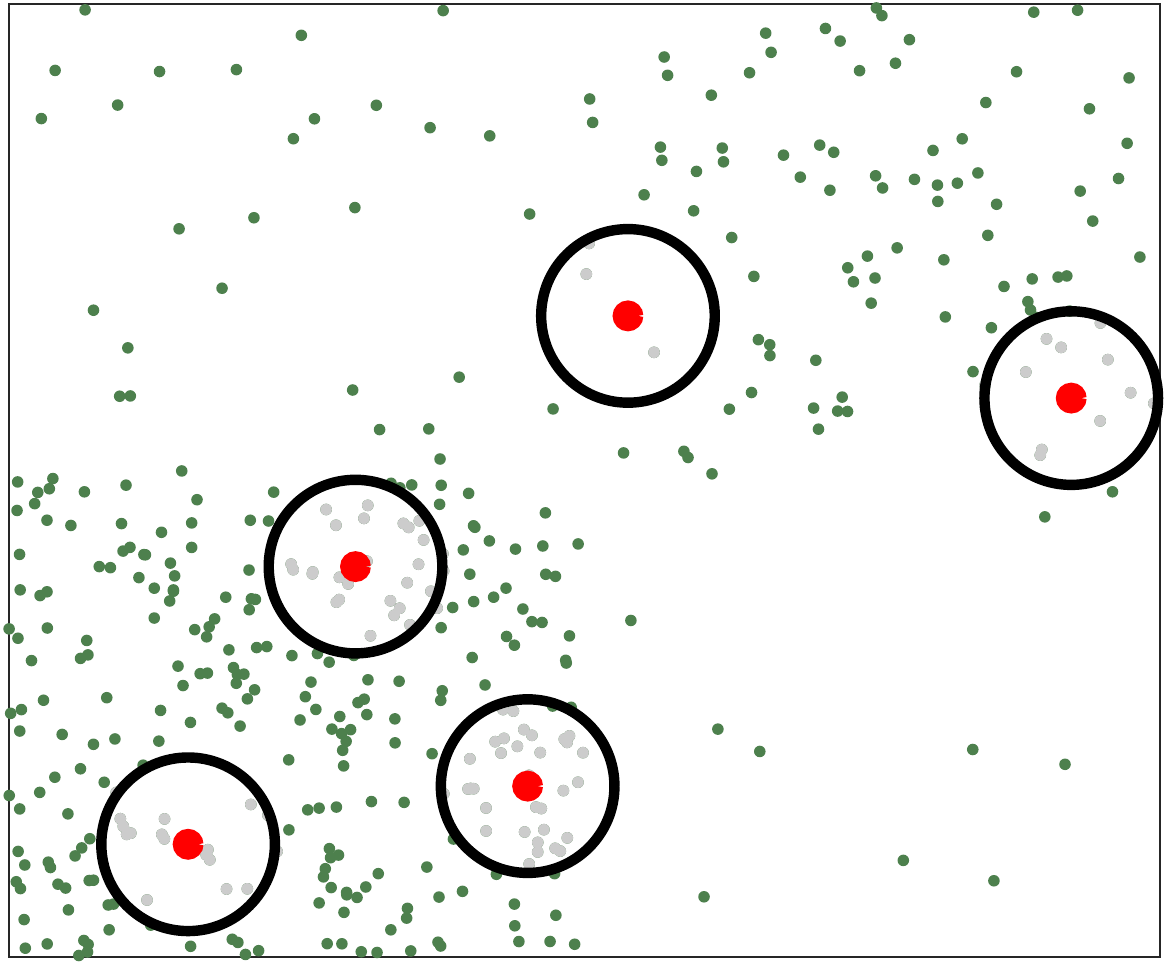}
  \caption{%\footnotesize 
  Demonstration of PDS.  
  The black circles of a certain radius $r$ mark regions claimed by the collected points (red). For the next iteration, if the newly sampled point falls in any of the circles (points colored in gray), this point will be rejected. 
  \vspace{-10pt}
  %\vspace{-25pt}
  }
  \label{fig:demo_PD}
  \end{wrapfigure}

PDS is one type of repulsive point process. It demonstrates stronger local repulsion compared to DPP \cite{biscio2016quantifying}.  Typically, it is implemented with the efficient dart throwing algorithm \cite{lagae2008comparison}, and provides point arrangements similar to those of DPP, albeit much more efficiently. 

This process dictates that the smallest distance between each pair of sample points should be at least $r$ with respect to some distance measure $D$. 
The second order product density $\varrho (x, y)$ for PDS is zero when the distance between two points are smaller than the disk radius $||x-y||\le r$, and converges to  $\varrho (x, y) = \lambda(x)\lambda(y)$ when the two points are far \cite{Oztireli12Analysis}. Thus, $\left[ \frac{\varrho(\bx, \mathbf{y})}{\lambda(\bx) \lambda(\mathbf{y})} - 1 \right] =-1 < 0$ when the points are within distance $r$, and $0$ when they are far. 

As demonstrated in Figure \ref{fig:demo_PD}, the basic dart throwing algorithm for PDS works as follows in each iteration: 1) randomly sample a data point; 2) if it is within a distance $r$ of any already accepted data point, reject; otherwise, accept the new point. 
We can also specify the maximum sampling size $k$. This means that we terminate the algorithm when $k$ points have been accepted.  
The computational complexity of PDS with dart throwing\footnote{In practice, the number of accepted points can be smaller than the number of considered points in the dataset, as some of them will be eliminated by not satisfying the distance criteria.} is $\mathcal{O}(k^2)$. This is much lower than the complexity $\mathcal{O}(N k^3)$ for k-DPP, where $N$ is the number of data points in the dataset.  
In the rest of the paper, we refer to this version of PDS as \textbf{``Vanilla PDS''}. 

\subsection{Poisson Disk Sampling with Adaptive Density}
\label{sec:GenerlizedPD}
To further utilize the potential benefit of our framework, we propose several variations of PDS. In particular, we use mingling index based marked processes. 
We then propose three variants as explained below: Easy PDS, where only the points far from decision boundaries repulse each other, as well as Dense PDS and Anneal PDS, where we can impose preferences on point densities. 

\paragraph{Mingling Index} 
The mingling index $\bar{M}_K(x_i)$ is defined as \cite{illian2008statistical}:
%\vspace{-5pt}
\begin{equation}
\bar{M}_K(x_i) = \frac{1}{K} \sum _{j=1}^K \mathbf{1} \big( m(x_i) \neq m(x_j) \big),
\end{equation}
where $m(x_i)$ indicates the mark of the point $x_i$. In case of a classification task, the mark is the class label. $\bar{M}_{K}(x_i)$  is the ratio of points with different marks than $x_i$ among its $K$ nearest neighbors. 

Depending on the mingling index, there are three different situations.
Firstly, if $\bar{M}_K(x_i)=0$, the region around $x_i$ only includes points from the same class. This makes $x_i$ a relatively easy point to classify. This type of points is preferred to be sampled in the early iterations for self-paced learning.
Secondly, if $\bar{M}_K(x_i) > 0$, this point may be close to a decision boundary. For variance reduction, we do not need to repulse this type of points. Additionally, sampling this type of points more often may help the model to refine the decision boundaries. 
Finally, if $\bar{M}_K(x_i)$ is very high, the point is more likely to be a hard negative. In this case, the point is mostly surrounded by points from other classes.
On a side note, points with high mingling indices can be interpreted as support vectors \cite{cortes1995support}.

\paragraph{Adaptive Variants of Poisson Disk Sampling~}
Gradients may change drastically when points are close to decision boundaries. Points in this region, thus, violate the assumption in Remark 3. Because of this, our first simple extension, which we call \textbf{``Easy PDS''}, sets the disk radius to $r_0$ when the point has a mingling index $M_{x_i} = 0$, and to $0$ if $M_{x_i} > 0$. This means that only easy points (with $M_{x_i} = 0$) repulse. On average, ``Easy PDS'' is expected to sample more of the difficult points compared with ``Vanilla PDS''. 

\begin{algorithm}[t]
   \caption{Draw throwing for Dense PDS}
   \label{alg:DT_hardexamples}
\begin{algorithmic}
   \STATE {\bfseries Input:} data $\bm{x}$, mini-batch size $S$,  mingling index  $M$ for all data,  the parameter for categorical distribution to sample mingling index $\pi$, disk radius $r_0$
   \REPEAT
\STATE  sample a mingling index $m \sim \text{Cat}(\pi)$  
\STATE  randomly sample a point $i$ with mingling index $m$
   \IF{ $x_i$ is not in disk of any samples} 
   \STATE insert $x_i$ to $B$
   \ENDIF
   \UNTIL{$S$ points have been accepted for the mini-batch}
\end{algorithmic}
\end{algorithm}
\setlength{\textfloatsep}{5pt}

For many tasks, when the data is highly structured, there are only few data points that are close to the decision boundary. To refine the decision boundary, inspired by hard example mining, we can sample points with a high mingling index more often. 
We thus propose the \textbf{``Dense PDS''} method summarized in Algorithm~\ref{alg:DT_hardexamples}. Instead of drawing darts randomly, we draw darts based on different mingling indices. The mingling indices can assume $K+1$ values, where $K$ is the number of nearest neighbors.  We thus can specify a parameter $\pi$ for a categorical distribution to sample mingling indices first. Then we randomly sample a dart which has the given mingling index. In this way, we can encode our preferred density with respect to the mingling index. 

It is straightforward to introduce an annealing mechanism in ``Dense PDS'' by using a different $\pi_n$ at each iteration $n$. Inspired by self-paced learning, we can give higher density to points with low mingling index in the early iterations, and slowly increase the density of points with high mingling index. 
We refer this method as \textbf{``Anneal PDS''}.

Note that all our proposed sampling methods only rely on the properties of the given data instead of any model parameters. Thus, they can be easily used as a pre-processing step, prior to the training procedure of the specific model. 
\section{Experiments}
\label{sec:exp}
We evaluate the performance of the proposed method in various application scenarios. Our methods show clear improvements compared to baseline methods in each case. 
We first demonstrate the behavior of different varieties of our method using a synthetic dataset. Secondly, we compare Vanilla PDS with DPP as in \cite{Zhang17Stochastic} on the Oxford flower classification task. Finally, we evaluate our proposed methods with two different tasks with very different properties: image classification with MNIST, and speech command recognition.

\begin{figure*}[t]
\centering
\subfigure[Random]{
\includegraphics[width = 0.23\textwidth,frame]{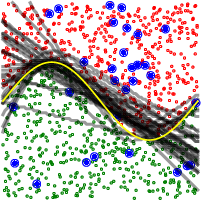}
\label{fig:toy_performance_30:random}
}
\subfigure[Vanilla PDS]{
\includegraphics[width = 0.23\textwidth,frame]{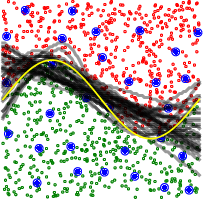}
\label{fig:toy_performance_30:vanilla}
}
\subfigure[Easy PDS]{
\includegraphics[width = 0.23\textwidth,frame]{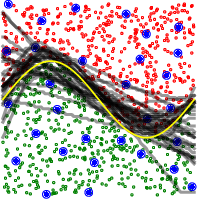}
\label{fig:toy_performance_30:easy}
}
\subfigure[Dense PDS]{
\includegraphics[width = 0.23\textwidth,frame]
{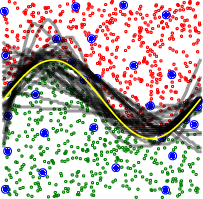}
\label{fig:toy_performance_30:dense}
}
%\vspace{-14pt}
\caption{Comparison of performance using one mini-batch sample distribution on synthetic data. Each experiment is repeated 30 times. The resulting decision boundaries are drawn as transparent black lines and the ground truth decision boundary is shown in yellow. As an example, the points from one of the sampled mini-batches are shown in blue.
%\vspace{-10pt}
}
\label{fig:toy_performance_30}
\end{figure*}

\begin{figure}[t]
\RawFloats
\centering
\begin{minipage}[c]{0.58\textwidth}
\centering
\subfigure[k=80]{
\includegraphics[width=3.8cm]{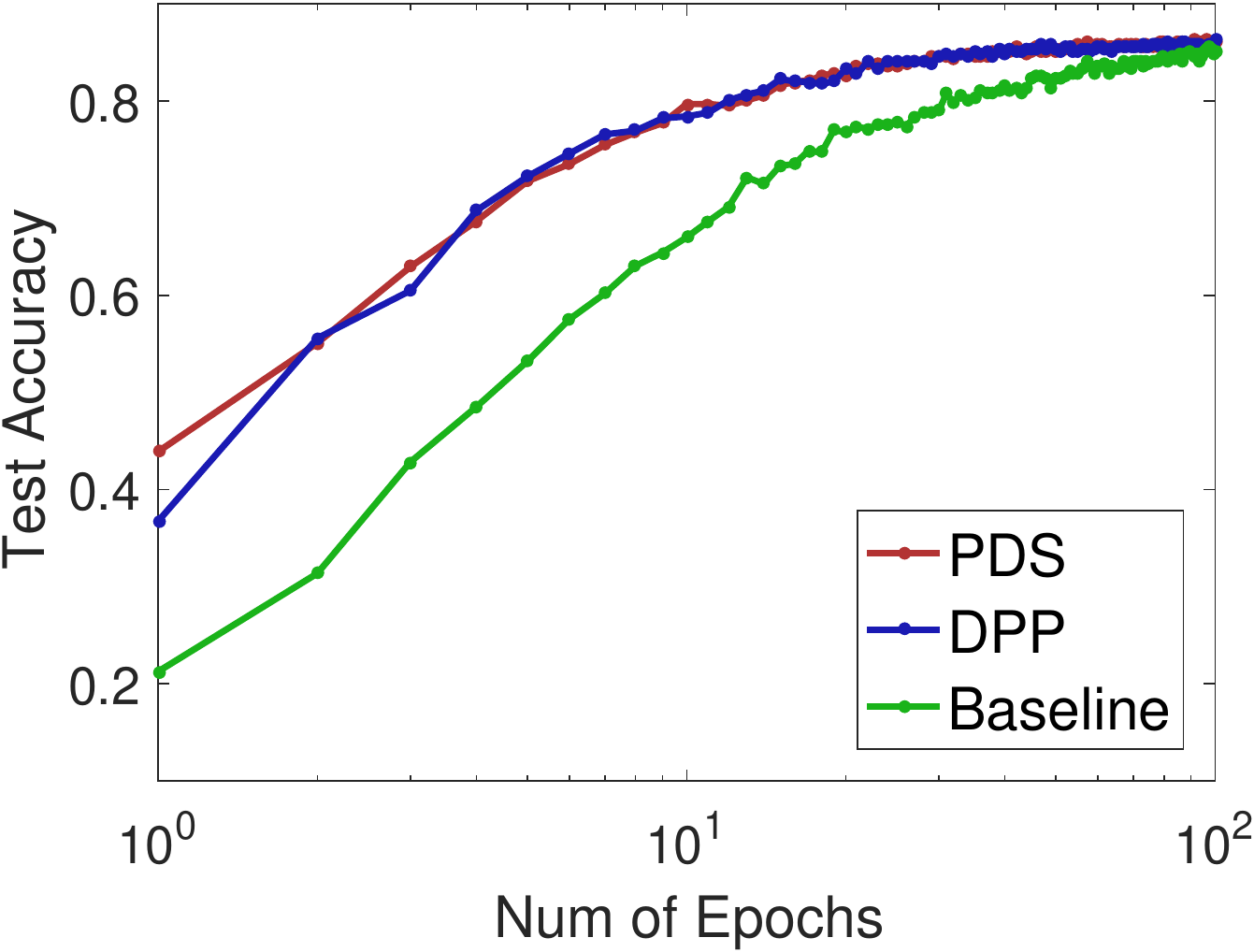}
}
\subfigure[k=150]{
\includegraphics[width=3.8cm]{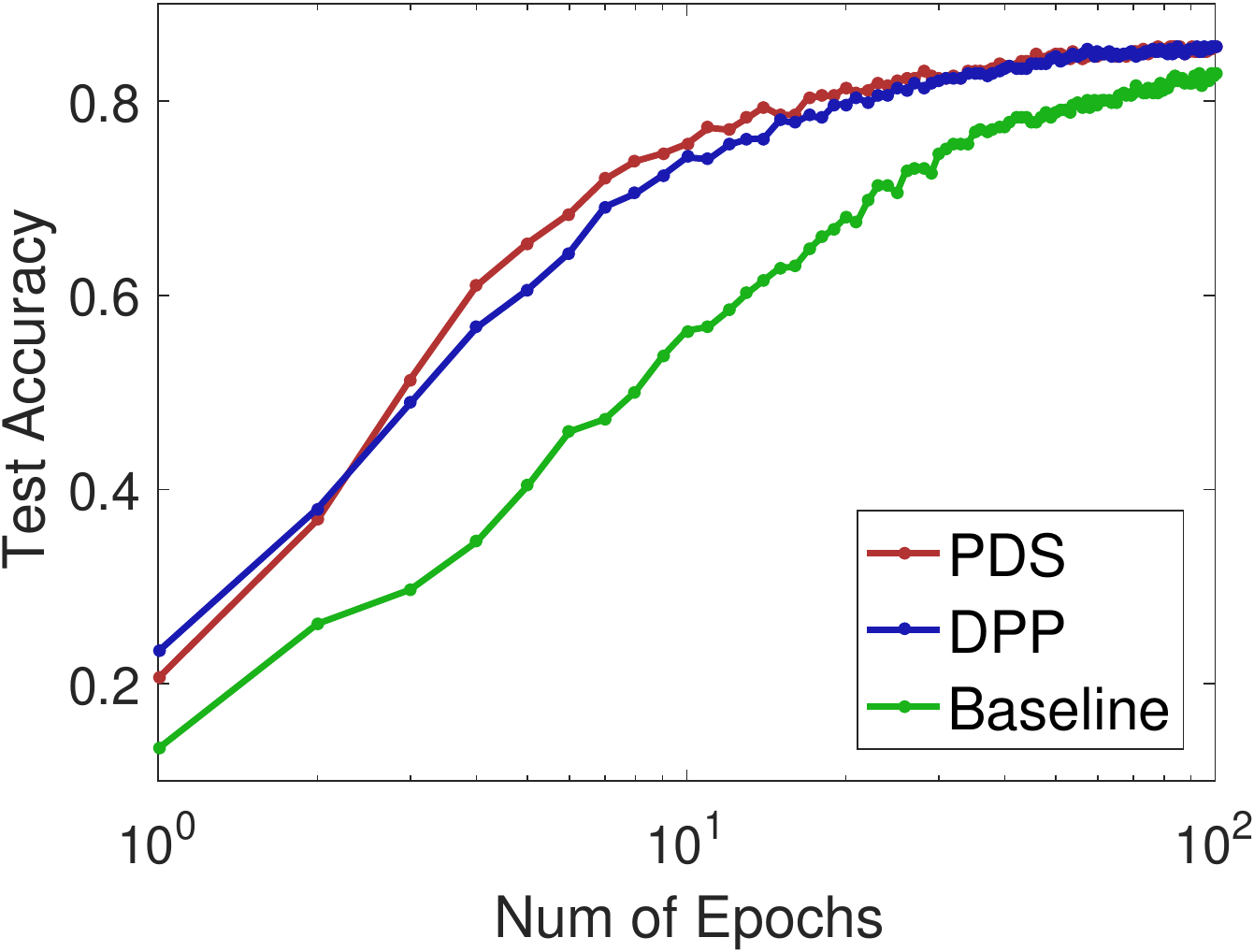}
}
\caption{ Oxford Flower Classification with Softmax. 
PDS has similar performance as DPP for sampling mini-batches. However PDS is more efficient as Table \ref{tab:sampletime} shows. Both methods converges faster than traditional SGD (Baseline).}
\label{fig:testAcc}
\end{minipage}
\hspace{0.015\textwidth}
\begin{minipage}[c]{0.38\textwidth}
\centering
\begin{tabular}{c c c}
\hline
k  & 80  & 150 \\
\hline
k-DPP & 29.687& 189.0303\\
Fast k-DPP  & 0.3312  & 1.8745 \\
PDS  &  0.0795  &0.1657\\
\hline
\end{tabular}
\captionof{table}{%\footnotesize
CPU time (sec) to sample one mini-batch. The disk radius $r$ is set to half of the median value of the distance measure $D$ for PDS. The table confirms that the computational complexity of PDS is $\mathcal{O}(k^2)$, which does not depend on the dataset size $N$, as opposed to $\mathcal{O}(Nk^3)$ for k-DPP. Thus, PDS is able to scale to massive datasets with large mini-batches. More results can be found in the appendix.}
\label{tab:sampletime}
\end{minipage}
%\vspace{-15pt}
\end{figure}

%%%%%%%%%%%%%
%Synthetic
%%%%%%%%%%%%%
\paragraph{Synthetic Data}
\label{sec:synthetic}
We evaluate our methods on two-dimensional synthetic datasets to illustrate the behavior of different sampling strategies. Figure~\ref{fig:toy_performance_30} shows two classes (green and red dots) separated by a wave-shaped (sine curve) decision boundary (yellow line). Intuitively, it should be favorable to sample diverse subsets and even more beneficial to give more weight to the data points at the decision boundary, i.e., sampling them more often. 
We sample one mini-batch with batch size $30$ using different sampling methods. 
For each method, we train a neural network classifier with one hidden layer of five units, using a single mini-batch. This model, albeit simple, is sufficient to handle the non-linear decision boundary in the example.

Figure \ref{fig:toy_performance_30} shows the decision boundaries by repeating the experiment 30 times. In order to illustrate the sampling schemes,  we also show one example of sampled mini-batch using blue dots.  
In the random sampling case (Figure~\ref{fig:toy_performance_30:random}), we can see that the mini-batch is not a good representation of the original dataset as some random regions are more densely or sparsely sampled. Consequently, the learned decision boundary is very different from the ground truth. 
Figures~\ref{fig:toy_performance_30:vanilla} shows Vanilla PDS.
Because of the repulsive effect, the sampled points cover the data space more uniformly and the decision boundary is improved compared to Figure~\ref{fig:toy_performance_30:random}.
In Figure~\ref{fig:toy_performance_30:easy}, we used Easy PDS, where the disk radius adapts with the mingling index of the points. 
We can see that points close to the decision boundary do not repel each other. This leads to a potentially more refined decision boundary as compared to Figure~\ref{fig:toy_performance_30:vanilla}.  
Finally, Dense PDS, shown in Figure~\ref{fig:toy_performance_30:dense}, chooses more samples close to the boundary  
and leads to the most precise decision boundary with a single mini-batch.

% %%%%%%%%%%%%%
% %Oxford Flower Old experiment
% %%%%%%%%%%%%%
\paragraph{Oxford Flower Classification}
\label{sec:oxford}

We compare our proposed PDS with DPP for mini-batch sampling on the fine-grained classification task as in \cite{Zhang17Stochastic} with the same experimental setting.  We use Vanilla PDS (with fixed disk radius) for fair comparison with DPP.
Figure~\ref{fig:testAcc} shows the test accuracy at the end of each training epoch. We see that sampling with either DPP or PDS leads to accelerated convergence as compared to traditional SGD. With similar performance as using DPP, PDS demonstrates significant improvement on sampling efficiency as shown in Table \ref{tab:sampletime}. More experimental results with different parameter settings are presented in the appendix.

%%%%%%%%%%%%%
%MNIST experiment
%%%%%%%%%%%%%
\paragraph{MNIST}
\label{sec:MNIST}

\begin{figure*}[t]
\centering
\hspace{-15pt}
\hspace{-15pt}
\subfigure[Mean Performance Comparison]{
\includegraphics[width = 0.47\textwidth]{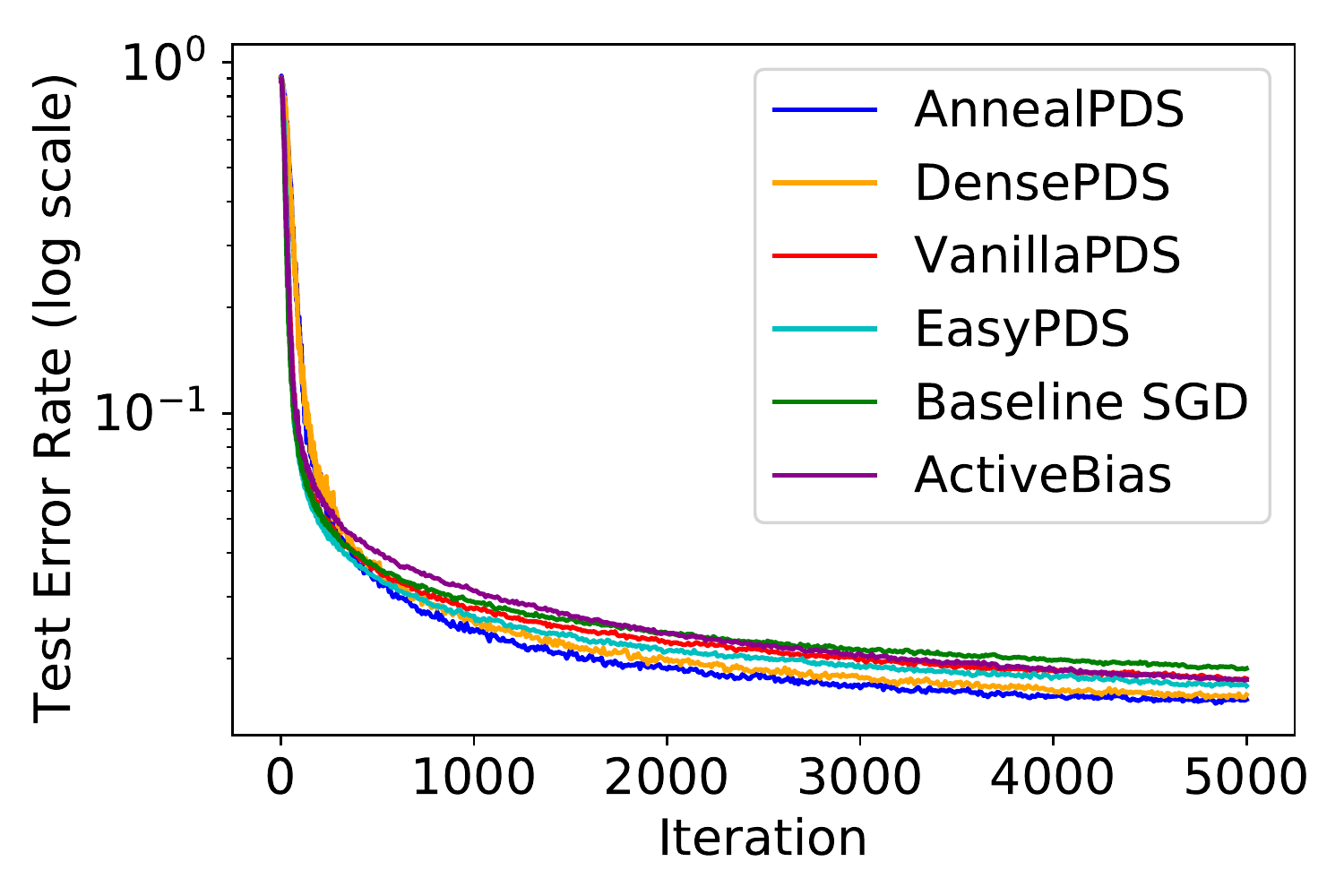}
\label{fig:MNIST:performance}
}
\hspace{-5pt}
\subfigure[Detailed Comparison]{
\includegraphics[width = 0.47\textwidth]{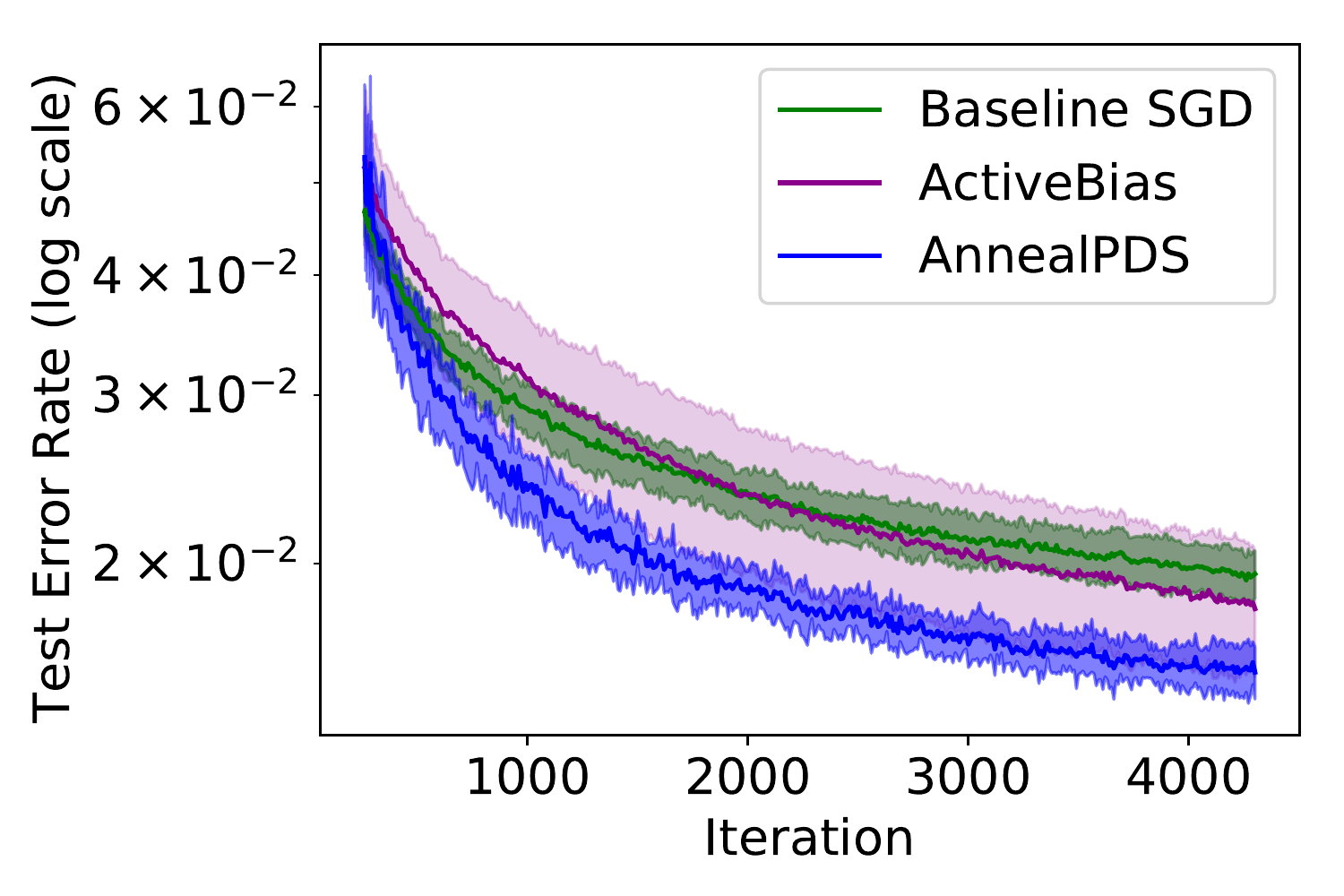}
\label{fig:MNIST:ThreeCompare}
}
\caption{% \footnotesize
MNIST experiment (10 repetitions). The mean performance for each method is reported in Panel (a). We compared all variations of our proposed methods with two baselines: Baseline SGD and ActiveBias \cite{chang2017active}. All our methods perform better than the baselines. AnnealPDS performs best. For better visualization, Panel (b) shows the mean and standard deviation of our proposed Anneal PDS comparing with two baselines closely.
}
\label{fig:MNIST}
\end{figure*}
We further show results for hand-written digit classification on the MNIST dataset \cite{lecun1998gradient}. We compare different variations of our method with two baselines: the traditional SGD and ActiveBias \cite{chang2017active}. We use half of the training data and the full test data.
As detailed in the appendix, with MNIST, data are well clustered and most data points have mingling index $0$ (easy to classify). 
A standard multi-layer convolutional neural network (CNN) from Tensorflow\footnote{ \url{https://www.tensorflow.org/versions/r0.12/tutorials/mnist/pros/}} is used in this experiment with standard experimental settings (details in appendix). 
 
Figure~\ref{fig:MNIST:performance} shows the test error rate evaluated after each SGD training iteration for different mini-batch sampling methods. All active sampling methods with PDS lead to improved final performance compared to traditional SGD.
Vanilla PDS clearly outperforms the baseline method. Easy PDS, performs very similarly to Vanilla PDS with slightly faster convergence in the early iterations. 
Dense PDS leads to better final performance at the cost of a slight decrease in initial convergence speed. The decision boundary is refined because we prefer non-trivial points during training.
Anneal PDS further improves Dense PDS with accelerated convergence in the early iterations. Figure \ref{fig:MNIST:ThreeCompare} shows the performance of test accuracy for Anneal PDS compared with baseline methods in a zoomed view.

We thus conclude that all different variations of PDS obtain better final performance, or conversely, achieve the baseline performance with fewer iterations.
With a proper annealing schedule to resemble self-paced learning, we can obtain even more improvement in the final performance.

%%%%%%%%%%%%%
%Speech Recognition experiment
%%%%%%%%%%%%%
\paragraph{Speech Command Recognition}
\label{sec:speech}

In this section, we evaluate our method on a speech command classification task 
as described in \cite{Sainath_2015_Interspeech}. 
The classification task includes twelve classes: ten isolated command words
, silence, or unknown class.

\begin{wrapfigure}{l}{0.4\textwidth}
\centering
%\vspace{-10pt}
\includegraphics[width = 5.5cm, height = 4.5cm]{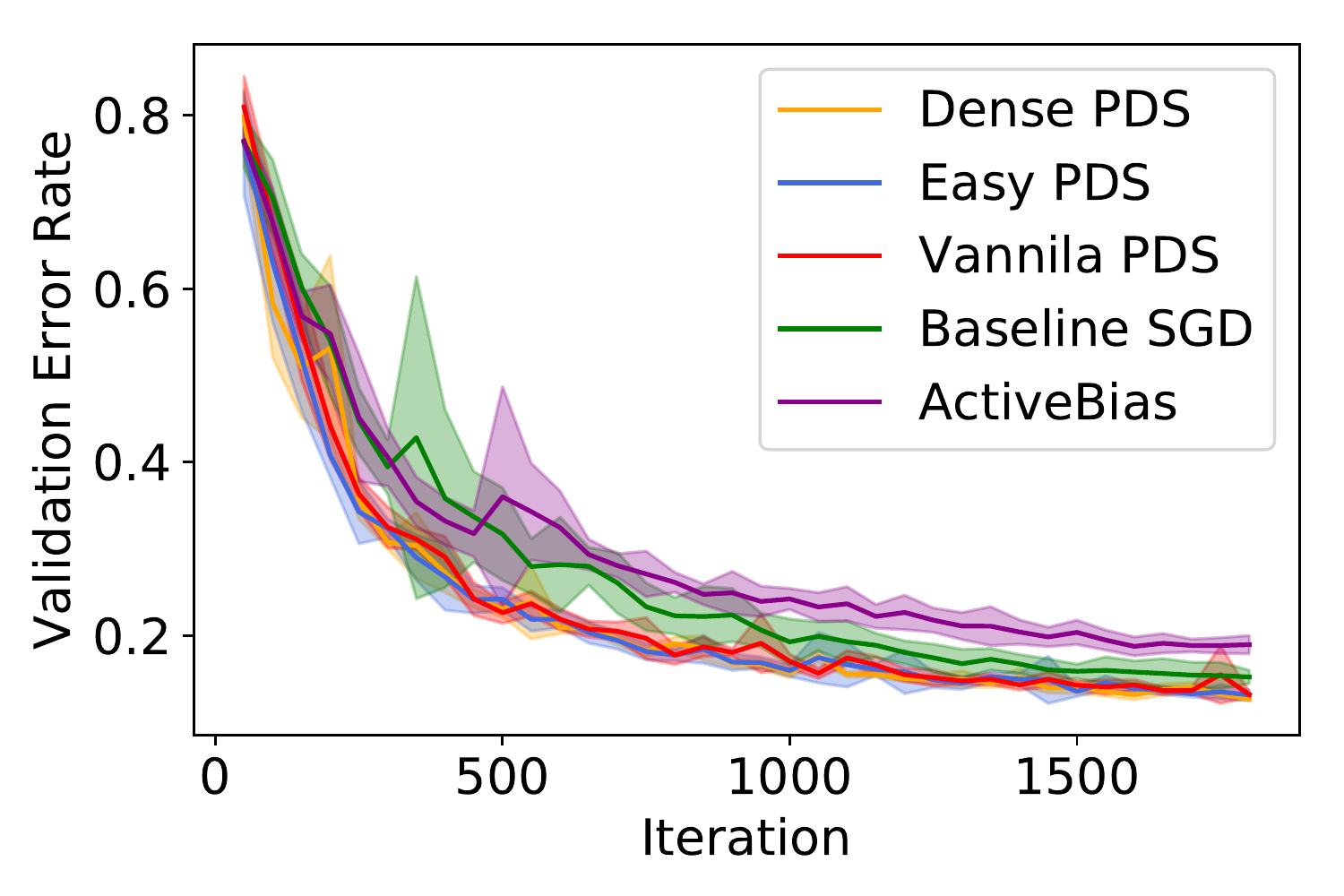}
%\vspace{-8pt}
\caption{Speech experiment (10 repetitions). We compare the performance over the validation set of different variations of PDS with traditional SGD for every 50 iterations. Means and standard deviations are shown.
\vspace{-5pt}
%\vspace{-15pt}
}
\label{fig:Speech:performance}
\end{wrapfigure}

The database consists of 64,727 one-second-long audio recordings.
As in \cite{Sainath_2015_Interspeech},  for each recording, 40 MFCC features \cite{DavisAndMarmelstein1980MFCC} are extracted at 10 msec time intervals resulting in $40\times 98$  features. 
We use the TensorFlow  implementation\footnote{\url{https://www.tensorflow.org/tutorials/audio_recognition}} with standard settings (see appendix for details).
Differently from the MNIST dataset, word classes are not clearly separated and data with different mingling index values are well distributed (see appendix).

Figure~\ref{fig:Speech:performance} shows the accuracy on the validation set evaluated every 50 training iterations. Using Vanilla PDS, the model converges with fewer iterations compared to the traditional random sampling of mini-batches. Easy PDS and Dense PDS show similar improvement since few data have $0$ mingling indices in this dataset.

As compared to the MNIST experiment, the gain of Vanilla PDS and Easy PDS is larger in this case since the dataset is more challenging. 
On the other hand, encouraging more difficult samples has a stronger impact on the MNIST dataset than in the Speech experiment. In all different settings, our mini-batch sampling methods are beneficial for both fast convergence and final model performance.  
\section{Discussion}
\label{sec:discussion}
In this work, we propose the use of repulsive point processes for active mini-batch sampling.  
We provide both theoretical and experimental evidence that using repulsive stochastic point processes can reduce the variance of stochastic gradient estimates, which leads to faster convergence. Additionally, our general framework also allows adaptive density and adaptive pair-wise interactions. This leads to further improvements in model performance thanks to balancing the information provided by the input samples, or enhancing the information around decision boundaries.

Our work is mainly focused on similarity measures in input space, which makes the algorithms efficient without additional run-time costs for learning. In future work, we will explore the use of our framework in the gradient space directly. This can potentially lead to even greater variance reduction, but at the same time may introduce higher computational complexity. Additionally, 
we believe that our proposed method may show even greater advantages in the two-stage framework such as Faster-R-CNN \cite{ren2015faster},
where the second stage network is trained using a subset of region proposals from the first stage. 

Finally, for sampling with adaptive density, we mainly use the information from mingling index, which can only be utilized for classification problems. In future work, we would also like to explore other measures such as sequences of annotations or graphs.

\bibliography{ref}
\bibliographystyle{plain}

\clearpage
\appendix
\onecolumn
\section{Derivation details}
\label{app:derivation}
\subsection{Derivation detail of variances of estimated Gradients}
\label{app:derivation_variances}
\paragraph{Preliminary: Campbell's Theorem}~
A fundamental operation needed for analyzing point processes is computing expected values of sums of sampled values from functions $\mathbb{E}_{\mathcal{P}} \left[ \sum{ f(\mathbf{x}_i) } \right]$, where the expectation is computed over different realizations of the point process $\mathcal{P}$, and $\mathbf{x}_1, \cdots$ are the points in a given realization. Note that the number of points in a given realization is also random, and hence is often omitted for the sums. This can also be generalized to functions of more than one variable. Campbell's theorem is a fundamental theorem that relates these sums to integrals of arbitrary functions $f$, and the product densities of $\mathcal{P}$. In particular, we will work with sums of the form $\sum{f(\mathbf{x}_i)}$ and $\sum_{ij}{f(\mathbf{x}_i, \mathbf{x}_j)}$. These are given by the following expressions:
\begin{equation}
\mathbb{E}_{\mathcal{P}} \left[ \sum{f(\mathbf{x}_i)} \right] = \int_{\mathbb{R}^d}{f(\mathbf{x}) \lambda(\mathbf{x}) d\mathbf{x}},
\label{eq:campbellOneVariable}
\end{equation}
 where  $f:\mathbb{R}^d \rightarrow \mathbb{R}$, and 
\begin{equation}
\mathbb{E}_{\mathcal{P}} \left[ \sum_{i \neq j}{f(\mathbf{x}_i, \mathbf{x}_j)} \right] = \int_{\mathbb{R}^d \times \mathbb{R}^d}{f(\mathbf{x}, \mathbf{y}) \varrho(\mathbf{x}, \mathbf{y}) d\mathbf{x} d\mathbf{y}},
\label{eq:campbellTwoVariables}
\end{equation}
for $f:\mathbb{R}^d \times \mathbb{R}^d \rightarrow \mathbb{R}$, and under the common assumption~\cite{illian2008statistical} that no two points in a process can be at the same location almost surely. In practice, we typically observe sample points in a finite domain $\mathcal{V} \in \mathbb{R}^d$. For such cases, the integration domains can be replaced by $\mathcal{V}$.

\paragraph{Derivation of Gradient Variances}~
For our derivations of bias and variance in this paper, we will need the first $\varrho^{(1)}$ and second $\varrho^{(2)}$ order product densities. The first order product density is given by $\varrho^{(1)}(\mathbf{x}) d\mathbf{x} = p(\mathbf{x})$. It can be shown that the expected number of points of $\mathcal{P}$ in a set $\mathcal{B}$ can be written as the integral of this expression: $\mathbb{E}_{\mathcal{P}} \left[ N(\mathcal{B}) \right] = \int_\mathcal{B}{\varrho^{(1)}(\mathbf{x}) d\mathbf{x}}$, where $N(\mathcal{B})$ is the (random) number of points of the process $\mathcal{P}$ in set $\mathcal{B}$. Hence, $\varrho^{(1)}(\mathbf{x})$ measures the local expected density of distributions generated by a point process. It is thus usually called the \emph{intensity} of $\mathcal{P}$ and denoted by $\lambda(\mathbf{x}) = \varrho^{(1)}(\mathbf{x})$. Pairwise correlations are captured by the second order product density $\varrho^{(2)}(\mathbf{x},\mathbf{y}) d\mathbf{x} d\mathbf{y} = p(\mathbf{x},\mathbf{y})$. The two statistics $\lambda(\mathbf{x})$ and $\varrho (\mathbf{x},\mathbf{y})$ are sufficient to exactly express bias and variance of estimators for integrals or expected values.

The expected value of the stochastic gradient $\mathbb{E}_{\mathcal{P}} \left[ \hat{\mathbf{G}}(\theta) \right]$, can be computed as follows
\begin{equation}
\mathbb{E}_{\mathcal{P}} \left[ \hat{\mathbf{G}}(\theta) \right] = \mathbb{E}_{\mathcal{P}} \left[ \frac{1}{k} \sum_i  \mathbf{g}(\mathbf{x}_i, \theta) \right] = \int_{\mathcal{V}}{\frac{1}{k} \mathbf{g}(\mathbf{x}, \theta) \lambda(\mathbf{x}) d\mathbf{x}}.
\label{eq:sgdExpectedValue}
\end{equation}

The $\lambda(\mathbf{x})$ is fundamentally related to the assumed underlying distribution $p_{\mathrm{data'}}(\mathbf{x})$ for the observed data-points. If the sampler does not introduce additional adaptivity, e.g. random sampling, then $\lambda(\mathbf{x})$ is directly proportional to $p_{\mathrm{data'}}(\mathbf{x})$.

%We can also compute the variance of the gradient estimate $\hat{\mathbf{G}}$ using the point processes framework. 
The variance of a multi-dimensional random variable can be written as $\mathrm{var}_{\mathcal{P}}(\hat{\mathbf{G}}) = \mathrm{tr}[\mathrm{cov}_{\mathcal{P}}(\hat{\mathbf{G}})] = \sum_{m} \mathrm{var}_{\mathcal{P}}(\hat{\mathbf{G}}_m)$, where $\hat{\mathbf{G}}_m$ denotes the $m^{th}$ component of $\hat{\mathbf{G}}$. The variance for each dimension is given by $\mathrm{var}_{\mathcal{P}}(\hat{\mathbf{G}}_m) = \mathbb{E}_{\mathcal{P}} [\hat{\mathbf{G}}_m^2] - ( \mathbb{E}[\hat{\mathbf{G}}_m] )^2$. These terms can be written in terms of $\lambda$ and $\varrho$ by utilizing Equations~\ref{eq:campbellTwoVariables}, and~\ref{eq:campbellOneVariable}, respectively, as follows

\begin{equation}
\begin{aligned}
\mathbb{E}_{\mathcal{P}} \left[\hat{\mathbf{G}}_m^2 \right] &= \mathbb{E}_{\mathcal{P}} \left[ \frac{1}{k^2}\sum_{ij}  \mathbf{g}_m(\mathbf{x}_i, \theta) \mathbf{g}_m(\mathbf{x}_j, \theta) \right] \\
&= \mathbb{E}_{\mathcal{P}} \left[\frac{1}{k^2} \sum_{i \neq j} \mathbf{g}_m(\mathbf{x}_i, \theta) \mathbf{g}_m(\mathbf{x}_j, \theta) \right]  
+ \mathbb{E}_{\mathcal{P}} \left[\frac{1}{k^2} \sum_{i} \mathbf{g}^2_m(\mathbf{x}_i, \theta) \right] \\
&= \frac{1}{k^2}\int_{\mathcal{V}\times \mathcal{V}}{  \mathbf{g}_m(\mathbf{x}, \theta)  \mathbf{g}_m(\mathbf{y}, \theta) \varrho(\mathbf{x}, \mathbf{y}) d\mathbf{x} d\mathbf{y}} \\
&+ \frac{1}{k^2}\int_{\mathcal{V}}{ \mathbf{g}^2_m(\mathbf{x}, \theta) \lambda(\mathbf{x}) d\mathbf{x}},
\label{eq:sgdVarianceTerm1}
\end{aligned}
\end{equation}
\begin{equation}
\begin{aligned}
\left( \mathbb{E}_{\mathcal{P}} \left[\hat{\mathbf{G}}_m \right] \right)^2 &= \left(\frac{1}{k} \int_{\mathcal{V}}{\mathbf{g}_m(\mathbf{x}, \theta) \lambda(\mathbf{x}) d\mathbf{x}} \right)^2 \\
&= \frac{1}{k^2} \int_{\mathcal{V} \times \mathcal{V}} { \mathbf{g}_m(\mathbf{x}, \theta) \mathbf{g}_m(\mathbf{y}, \theta) \lambda(\mathbf{x}) \lambda(\mathbf{y}) d\mathbf{x} d\mathbf{y}}
\label{eq:sgdVarianceTerm2}
\end{aligned}
\end{equation}
Finally, summing over all dimensions, we can get the following formula for the variance of $\hat{\mathbf{G}}$
\begin{equation}
\begin{aligned}
\mathrm{var}_{\mathcal{P}} (\hat{\mathbf{G}}) =&\frac{1}{k^2}\int_{\mathcal{V}\times \mathcal{V}}{  \lambda(\mathbf{x})  \lambda(\mathbf{y}) \mathbf{g}(\mathbf{x}, \theta)^T  \mathbf{g}(\mathbf{y}, \theta) \left[ \frac{\varrho(\mathbf{x}, \mathbf{y})}{\lambda(\mathbf{x}) \lambda(\mathbf{y})} - 1 \right] d\mathbf{x} d\mathbf{y}} \\
&+ \frac{1}{k^2} \int_{\mathcal{V}}{ \Vert \mathbf{g}(\mathbf{x}, \theta) \Vert^2 \lambda(\mathbf{x}) d\mathbf{x}}.
\label{eq:sgdVariance_long}
\end{aligned}
\end{equation}

\section{Additional Results}
\label{app:examples}
Figure \ref{fig:sample_exp} shows examples of samples using DPP and using Poisson disk sampling (by dart throwing). 
\begin{figure}[h]
\centering
\subfigure[DPP exp 1]{
\includegraphics[width=3.4cm]{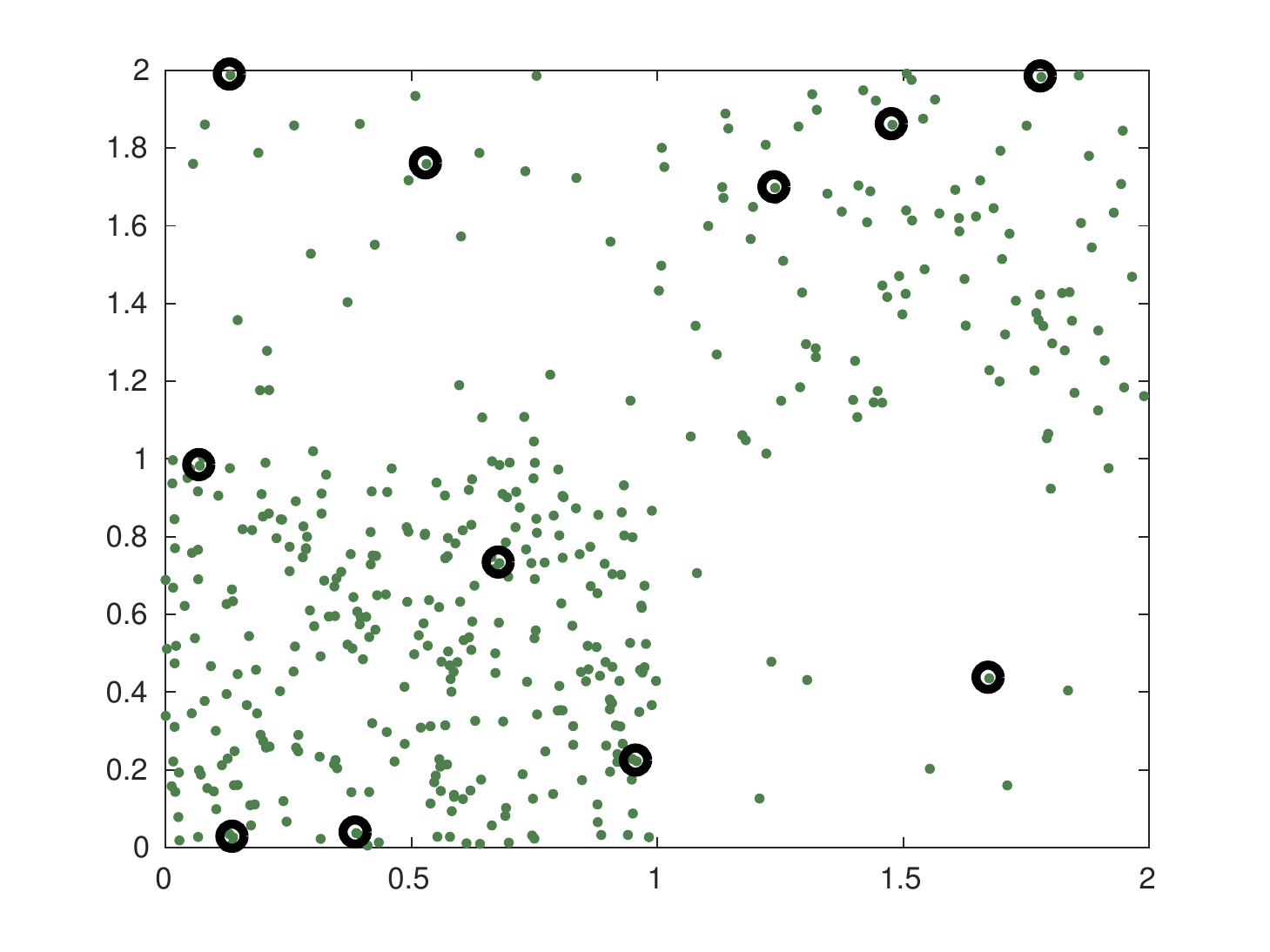}
}
\hspace{-10pt}
\subfigure[DPP exp 2]{
\includegraphics[width=3.4cm]{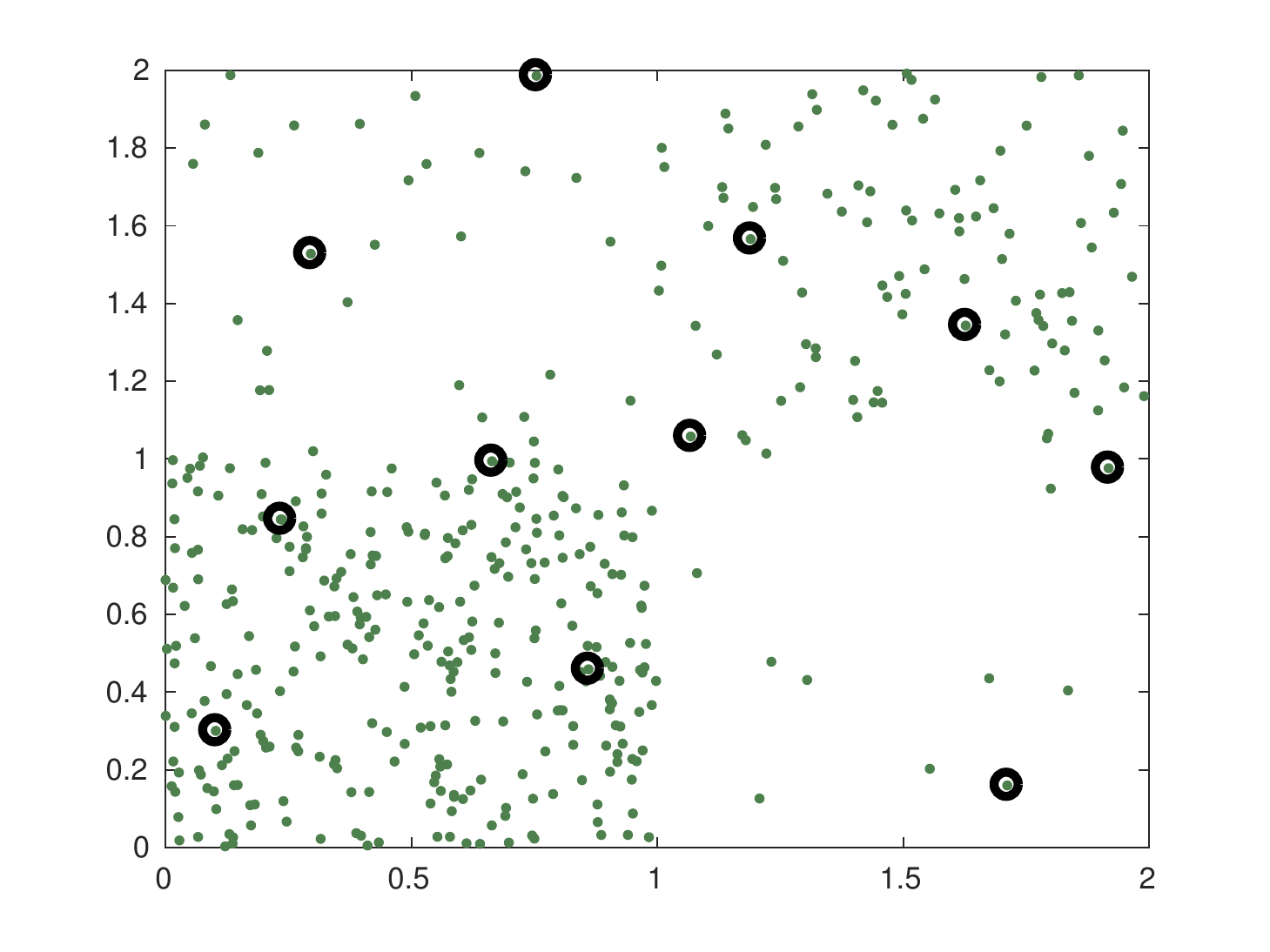}
}
\hspace{-10pt}
\subfigure[DPP exp 3]{
\includegraphics[width=3.4cm]{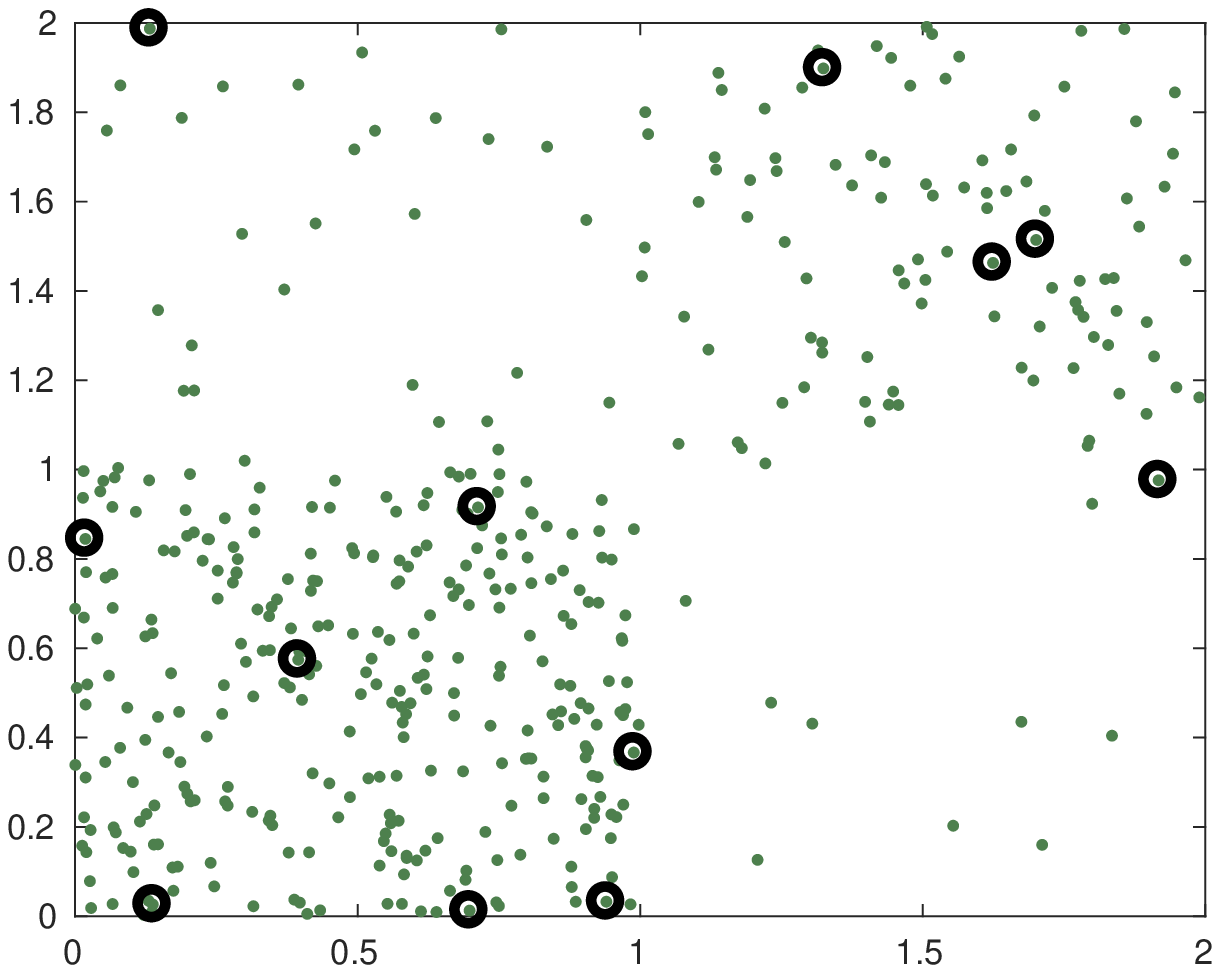}
}
\hspace{-10pt}
\subfigure[DPP exp 4]{
\includegraphics[width=3.4cm]{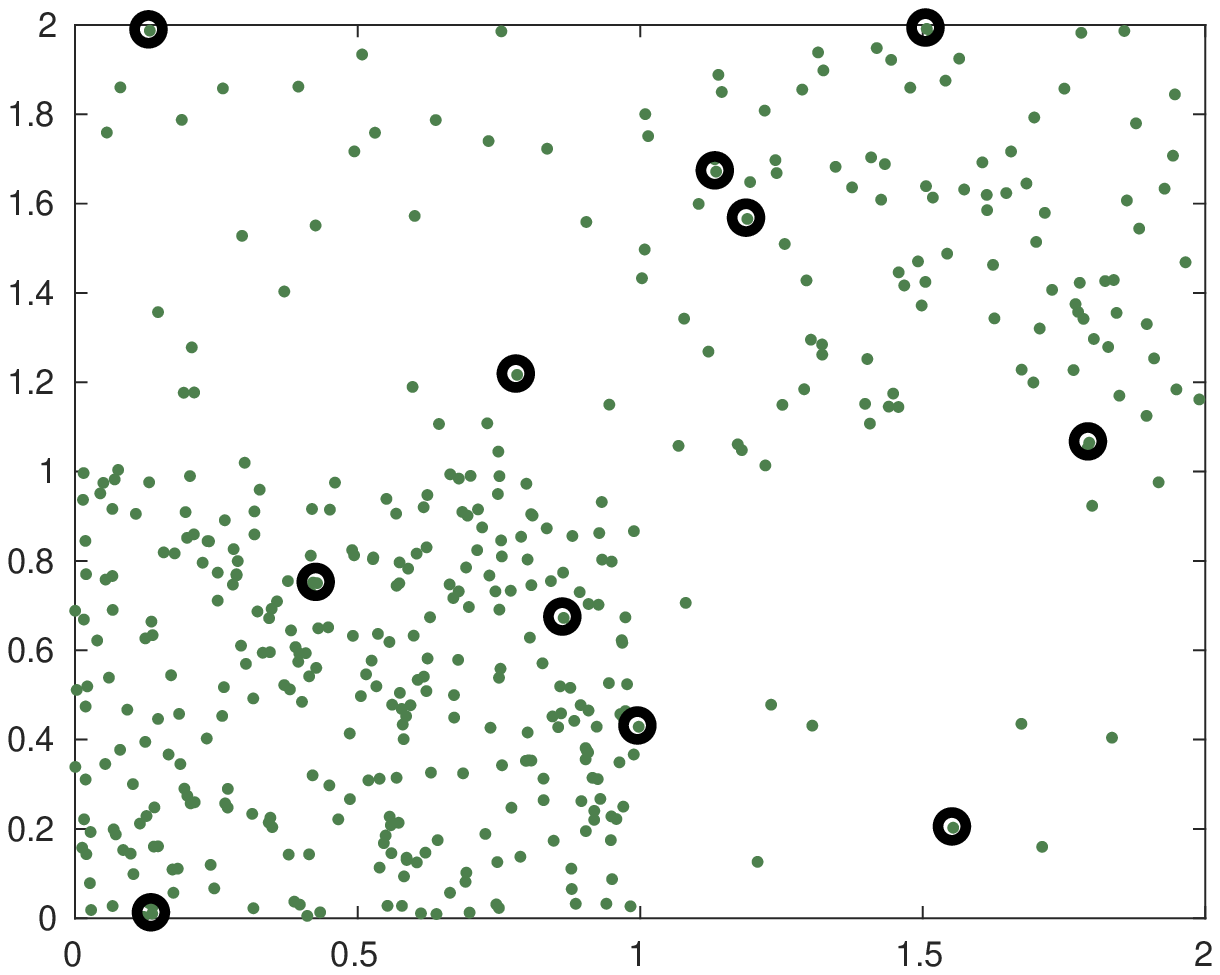}
}

\subfigure[Poisson Disk exp 1]{
\includegraphics[width=3.3cm]{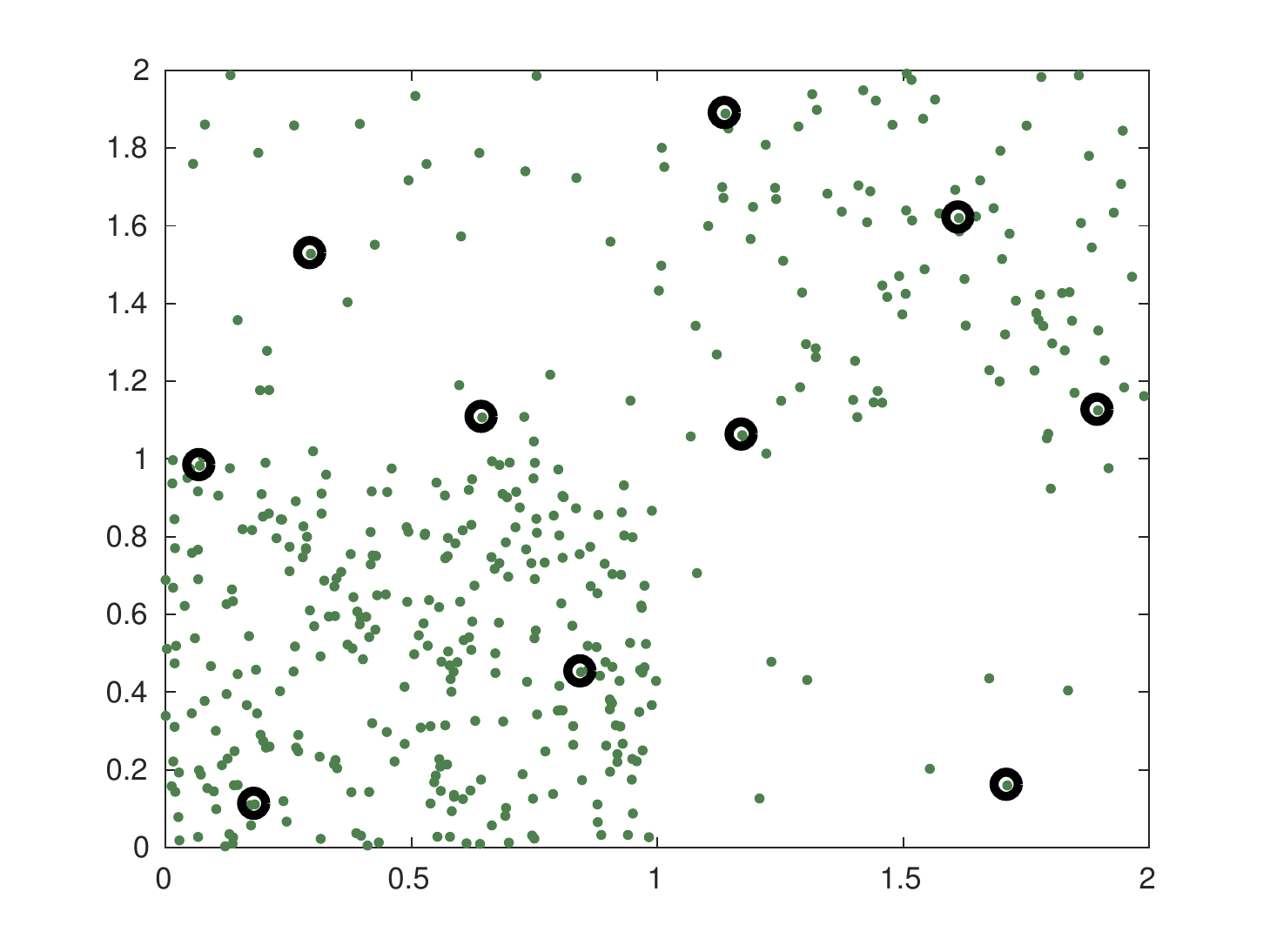}
}
\hspace{-10pt}
\subfigure[Poisson Disk exp 2]{
\includegraphics[width=3.3cm]{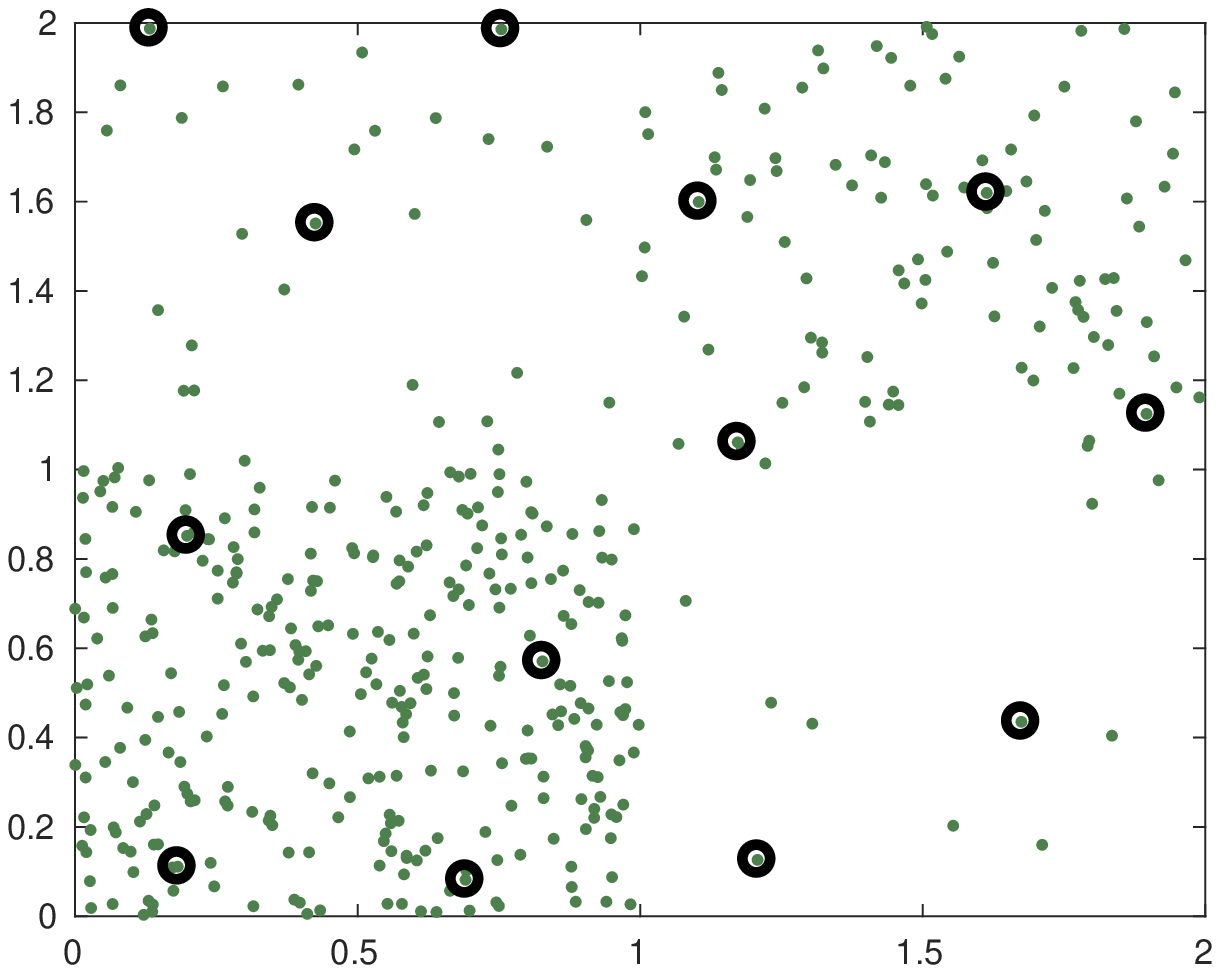}
}
\hspace{-10pt}
\subfigure[Poisson Disk exp 3]{
\includegraphics[width=3.3cm]{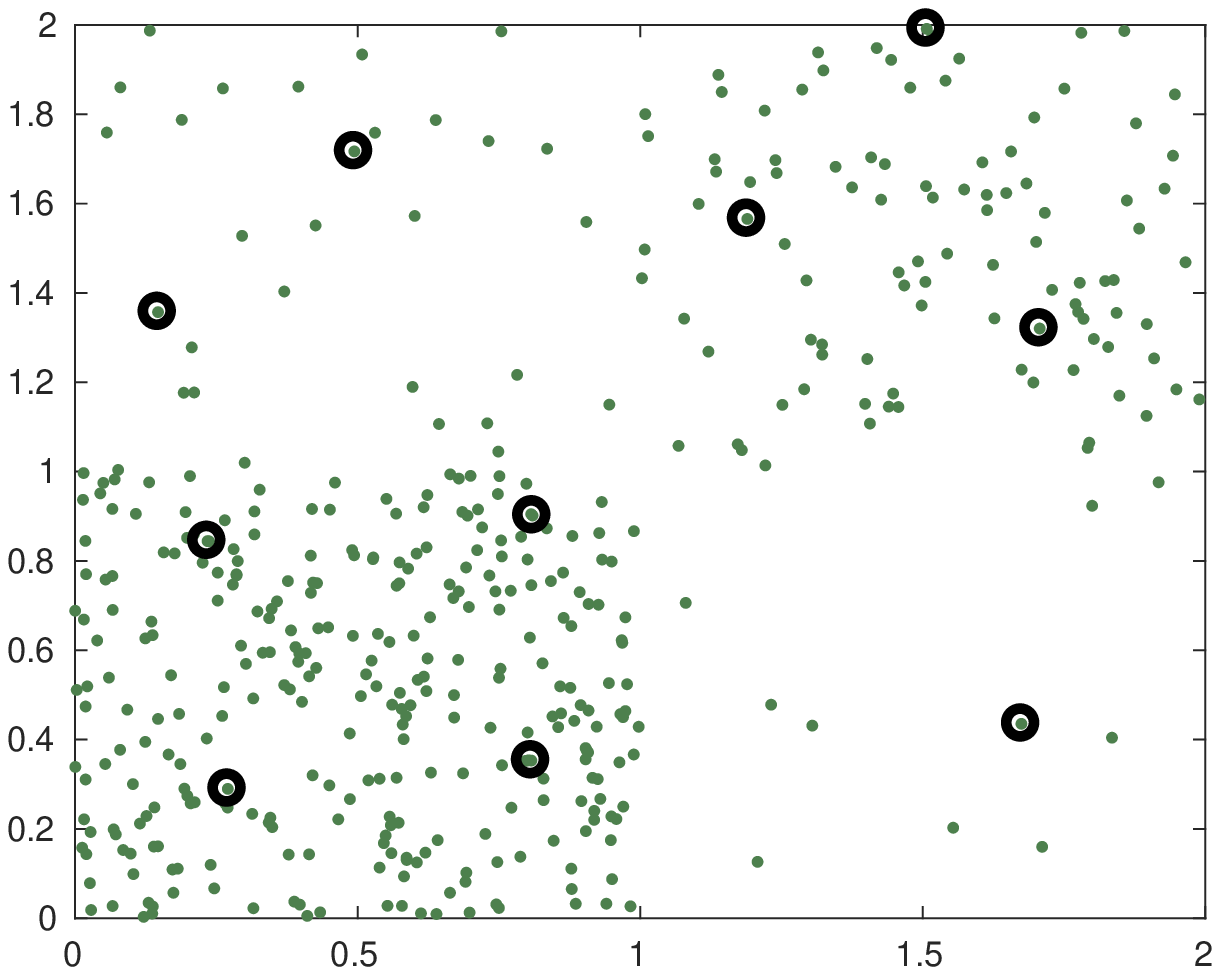}
}
\hspace{-10pt}
\subfigure[Poisson Disk exp 4]{
\includegraphics[width=3.3cm]{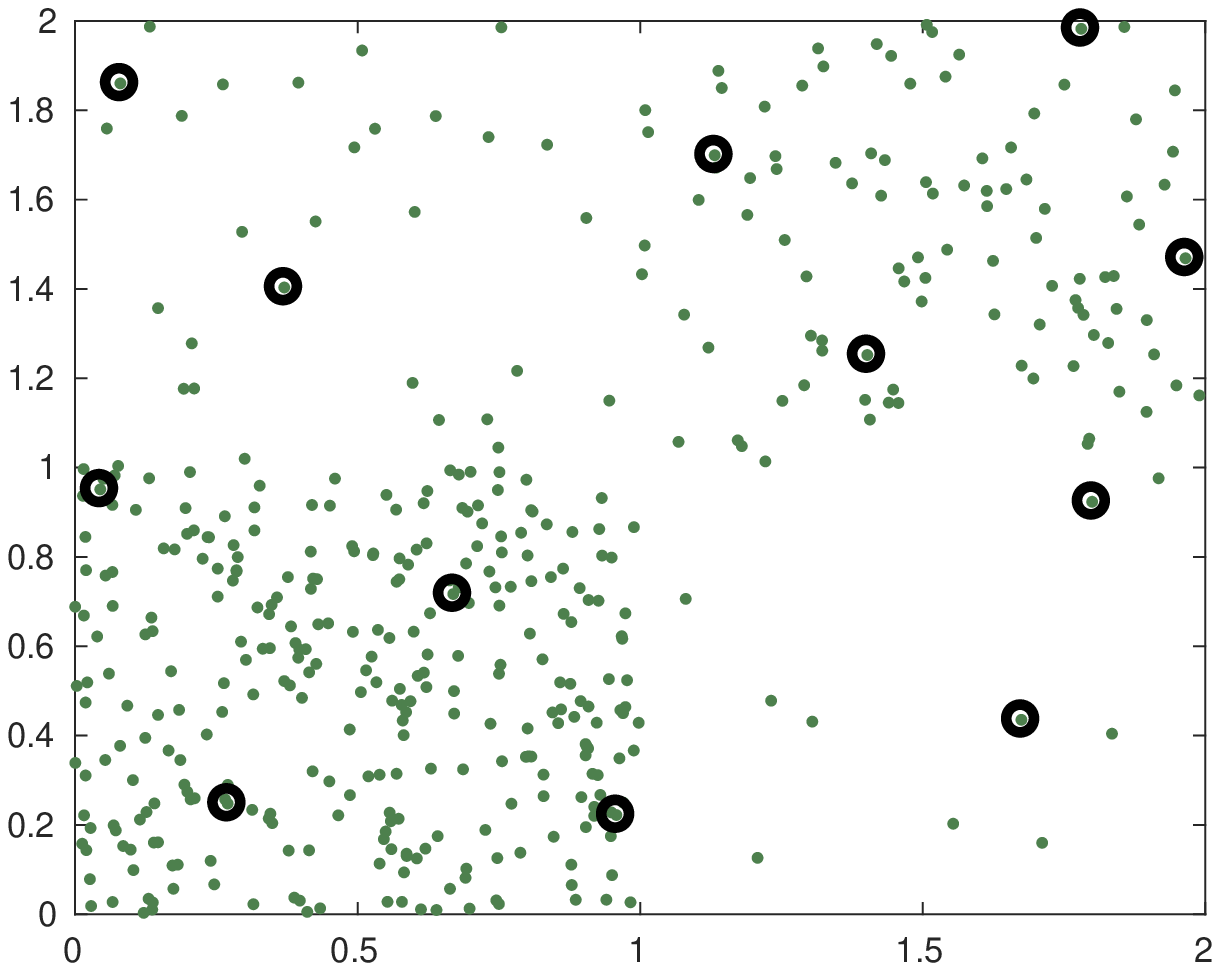}
}
\caption{\footnotesize Examples of sampled data points. The first row are sampled using DPP and the second row are sampled with Poisson Disk sampling. The green dots are the training data and the black circles indicated that the data point has been sampled.}
\label{fig:sample_exp}
\end{figure}

For Oxford Flower classification experiments, we compare all different experimental settings with \cite{Zhang17Stochastic} here. Figure \ref{fig:FlowerTestAcc_DPP} presents the original results from \cite{Zhang17Stochastic} with four different mini-batch sizes. The feature space is constructed by concatenating the off-the-shelf image feature and one-hot-vector label feature. With such concatenation, the $w$ is the weight of label part and (1-$w$) is the weight of off-the-shelf feature part. We perform the experiments under those different settings with our proposed PDS.  Figure \ref{fig:testAcc_appendix}  shows the performance (test accuracy at the end of each training epoch) using Vanilla PDS. Our proposed method with Vanilla PDS exhibits similar performance as DPP in all different experimental settings (as expected), but much more efficiently (as discussed in the paper). However, as shown in the Table 1 in the paper, DPS is much more efficient than DPP. To be noticed,  in the main paper, the results are with $w=0.7$.    
%As a comparison to our result with efficient Poisson Disk Sampling in Figure \ref{fig:testAcc} in the paper, we provided  the original result from \cite{Zhang17Stochastic} for the same experiment setting here in the supplement as shown in Figure \ref{fig:FlowerTestAcc_DPP}.

\begin{figure}[t]
\begin{center}
\subfigure[ k=50%, Top3: 0.9, 0.7, 1;\protect\newline 
%Best:$86.7\%$  Baseline:$84.7\%$
]{
\includegraphics[width=3.4cm, height = 2.85cm]{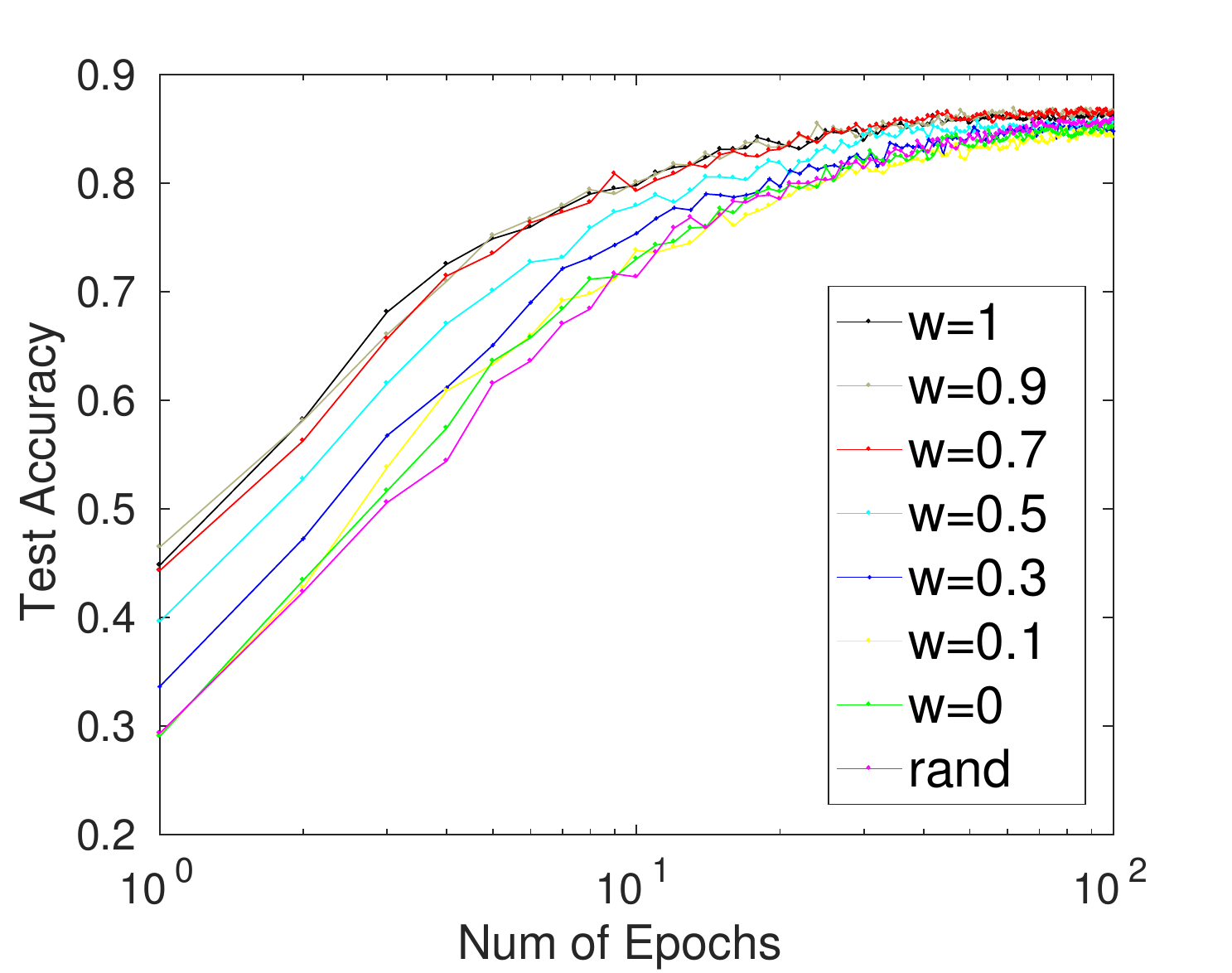}
}
\hspace{-10pt}
\subfigure[ k=80%, Top3: 0.9, 0.7, 1;  \protect\newline
%Best:$86.7\%$ Baseline:$81.8\%$
]{
\includegraphics[width=3.3cm, height = 2.85cm]{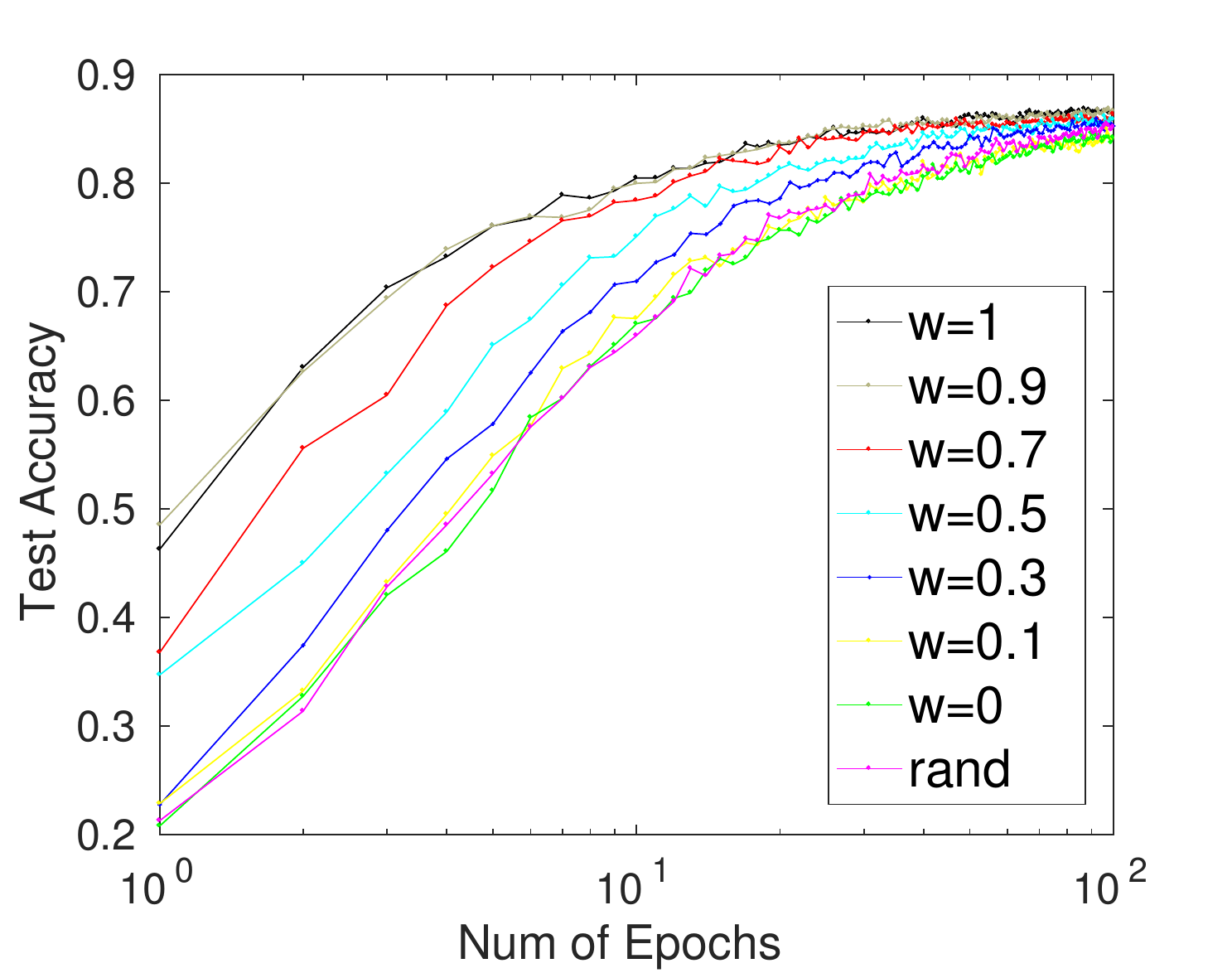}
}
\vspace{-8pt}
\subfigure[ k=102%, Top3: 1, 0.9, 0.7; \protect\newline
%Best:$86.5\%$ Baseline:$84.5\%$
]{
\includegraphics[width=3.3cm, height = 2.85cm]{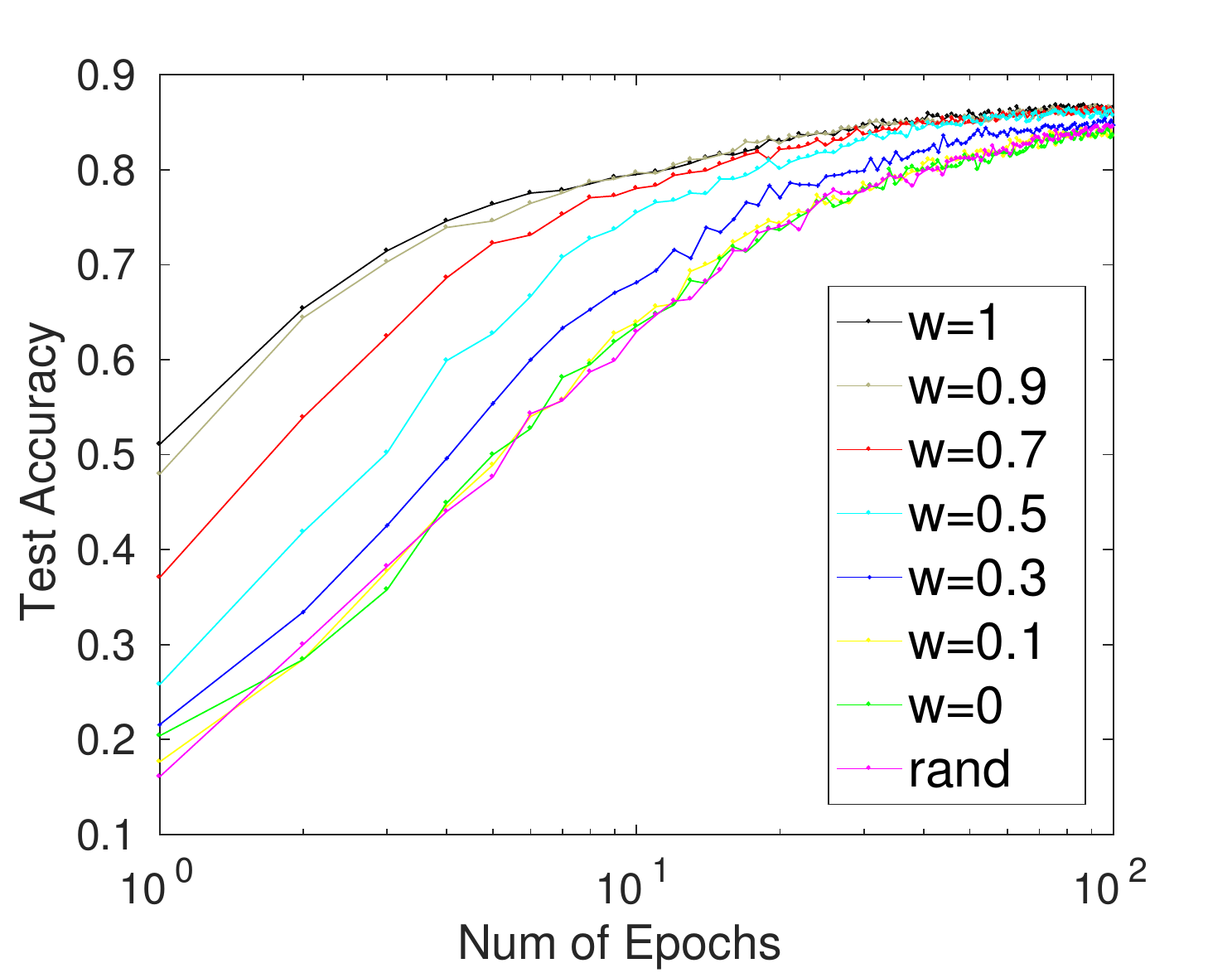}
}
\hspace{-10pt}
\subfigure[k=150%, Top3: 0.7, 0.5, 0.9; \protect\newline
%Best:$85.5\%$  Baseline:$83.1\%$
]{
\includegraphics[width=3.3cm, height = 2.85cm]{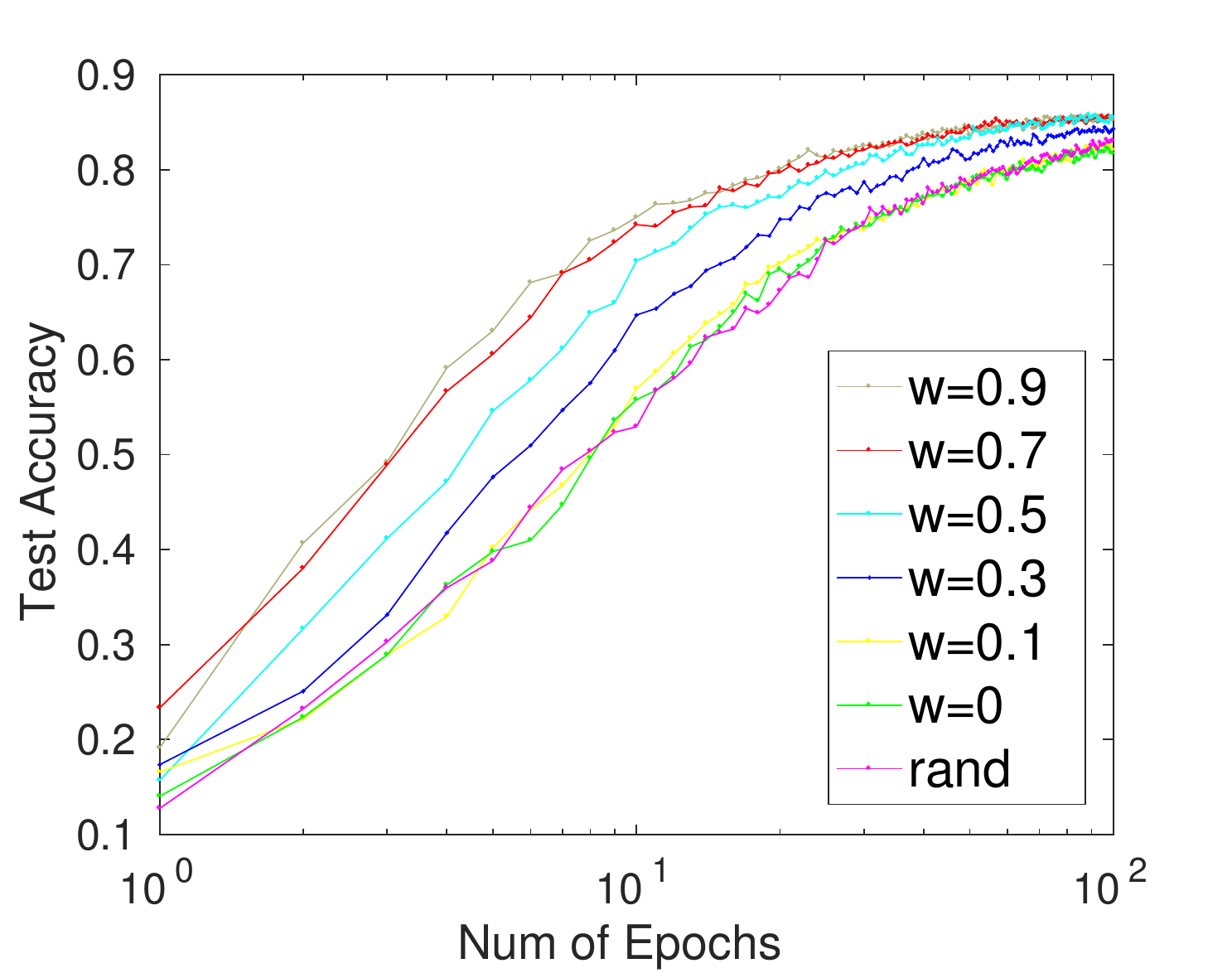}
}
\caption{Result of Oxford Flower Classification with DPP from \cite{Zhang17Stochastic}.} 
\vspace{-8pt}
\label{fig:FlowerTestAcc_DPP}
\end{center}
\end{figure}

\begin{figure*}[t]
\centering
\subfigure[k=50]{
\includegraphics[width=0.25\textwidth]{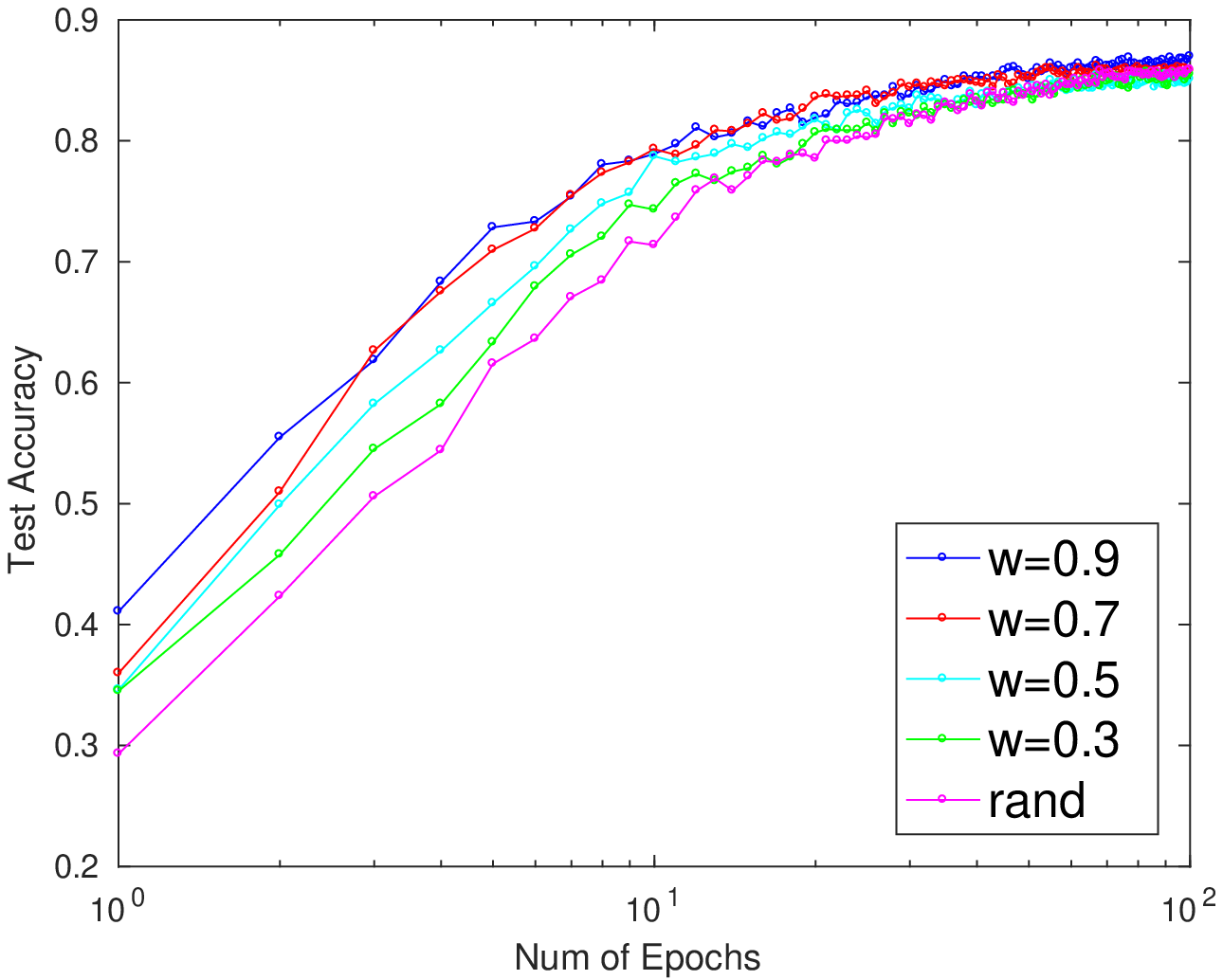}
}
\hspace{-19pt}
\subfigure[k=80]{
\includegraphics[width=0.25\textwidth]{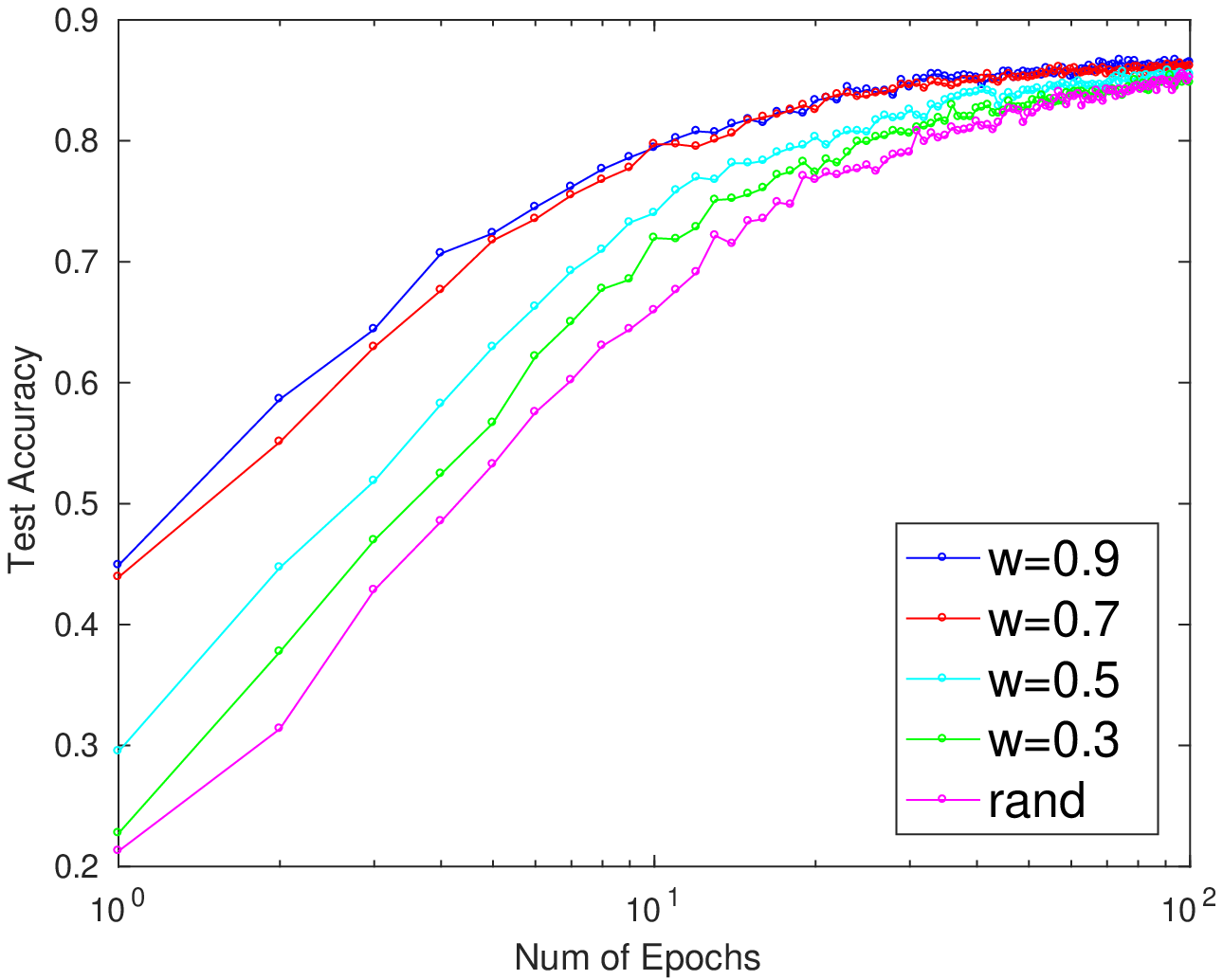}
}
\hspace{-19pt}
\subfigure[k=102]{
\includegraphics[width=0.25\textwidth]{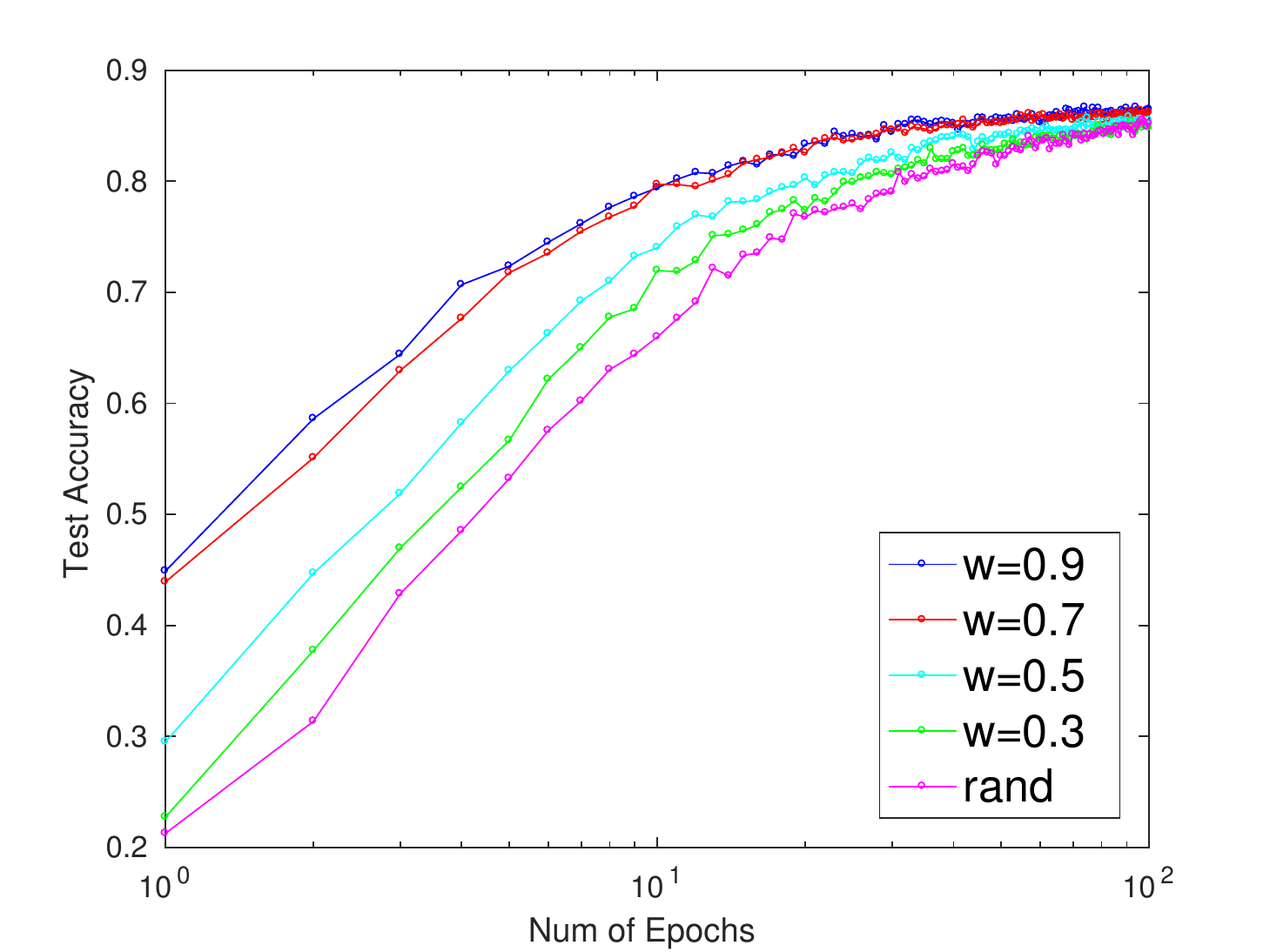}
}
\hspace{-19pt}
\subfigure[k=150]{
\includegraphics[width=0.25\textwidth]{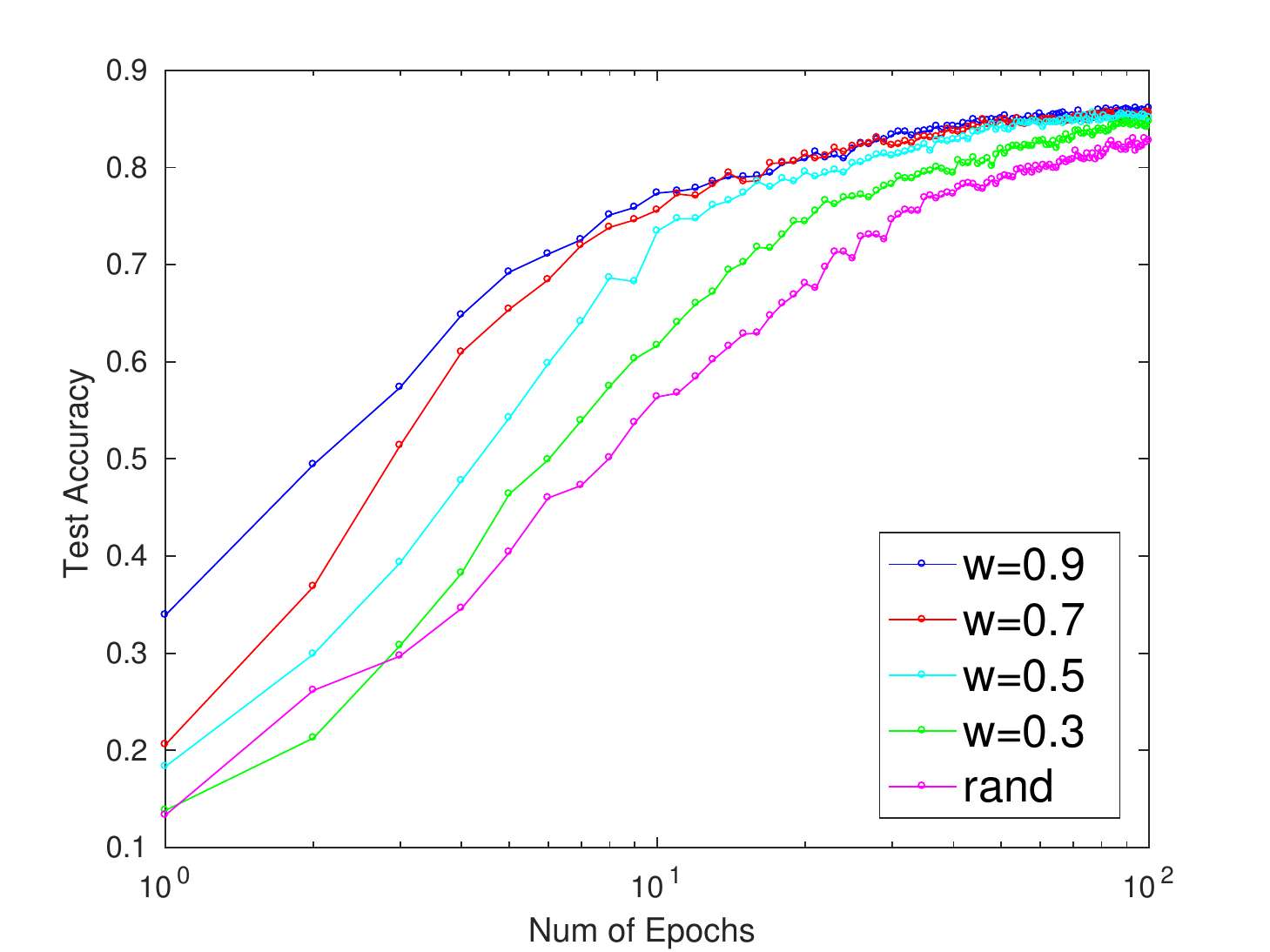}
}
\vspace{-8pt}
\caption{ Oxford Flower Classification with Softmax using PDS. The performance is similar as using DPP Figure \ref{fig:FlowerTestAcc_DPP}.}  %SGD with PDS converges faster than traditional SGD (rand). The $w$ is the weight balancing the visual feature and label information as in \cite{Zhang17Stochastic}.
\label{fig:testAcc_appendix}
\end{figure*}

\begin{table}[h]
\centering
\begin{tabular}{c c c c c c}
\hline
k & 50 & 80 & 102 & 150 & 200 \\
\hline
k-DPP time & 7.1680 & 29.687& 58.4086 & 189.0303 & 436.4746\\
Fast k-DPP & 0.1032 & 0.3312 & 0.6512 & 1.8745 & 4.1964\\
Poisson Disk & 0.0461 &  0.0795 &  0.1048 &0.1657& 0.2391\\
\hline
\end{tabular}
\caption{Oxford Flower Classification experiment. CPU time (sec) to sample one mini-batch.
%Sampling time for each mini-batch. CPU time is reported
%in the unit of seconds. %In practice, the operation can be easily parallelized for all these methods.
%\vspace{-10pt}
}
\label{tab:sampletime_appendix}
%\vspace{-5pt}
\end{table}

\begin{figure}
\subfigure[Data Visualization for MNIST]{
\includegraphics[width = 0.48\textwidth, height = 5cm]{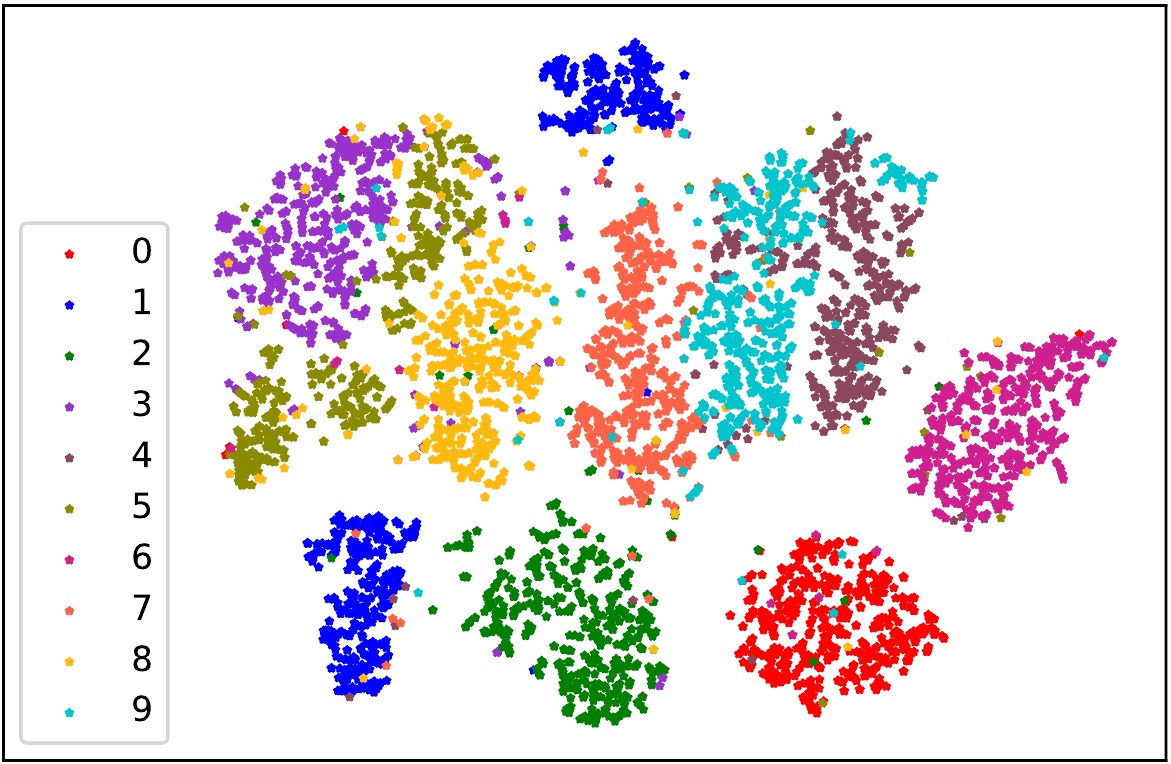}
\label{fig:MNIST:tSNE}
}
\subfigure[Mingling index distribution]{
\includegraphics[width = 0.48\textwidth, height = 5cm]{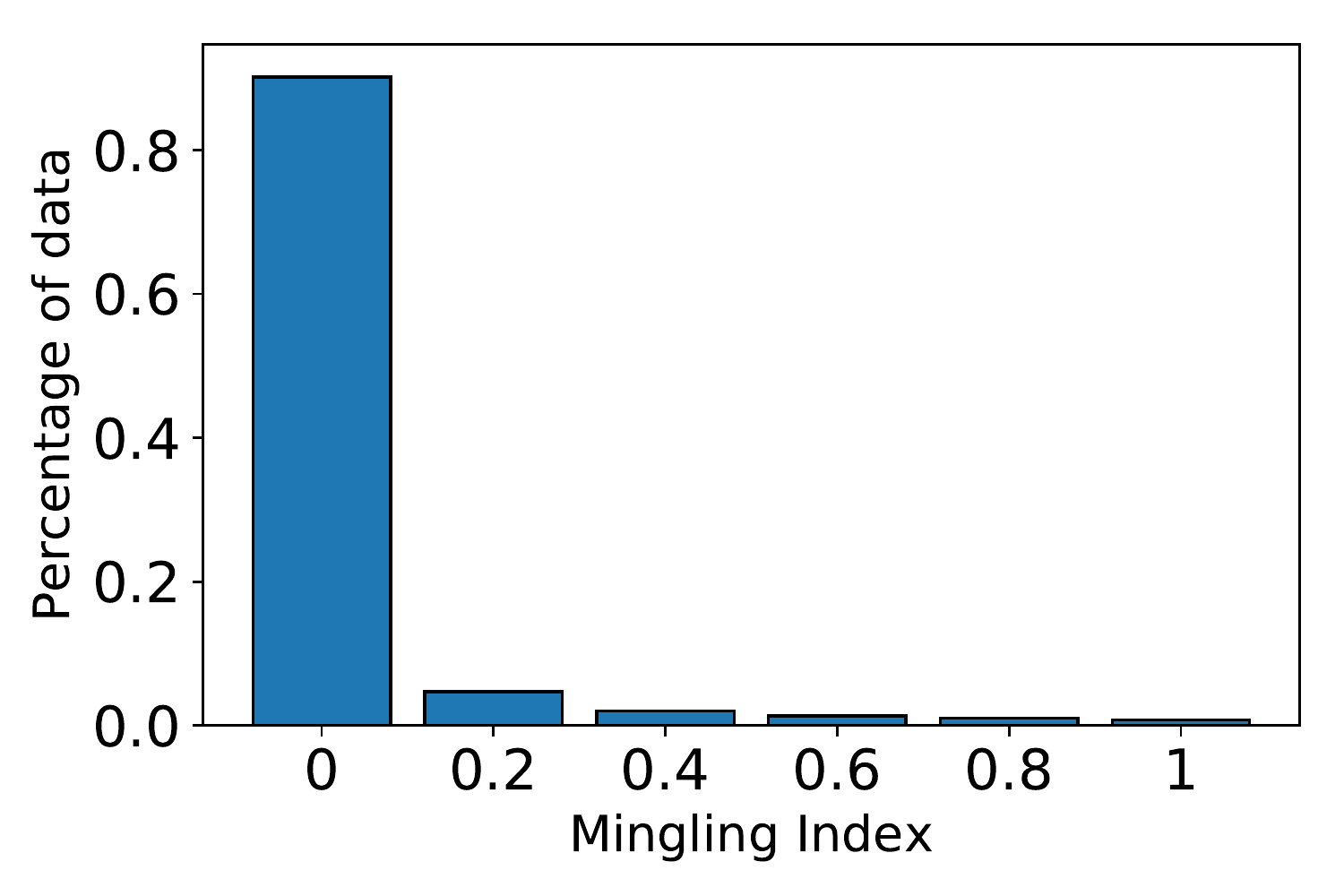}
\label{fig:MNIST:mingling}
}
\caption{MNIST experiment. Panel (a) visualizes the dataset using tSNE (5000 data points from training data are used for visualization). Panel (b) shows the distribution of data with different mingling index values.}
\end{figure}

\begin{figure}
\subfigure[Mean Performance Comparison]{
\includegraphics[width = 0.45\textwidth]{MNIST/MNIST_Compare_Five_PD_log_bigFont_PDS_mean_randOrder.pdf}
\label{fig:MNIST:performance_mean_app}
}
\subfigure[ActiveBias (baseline)]{
\includegraphics[width = 0.45\textwidth]{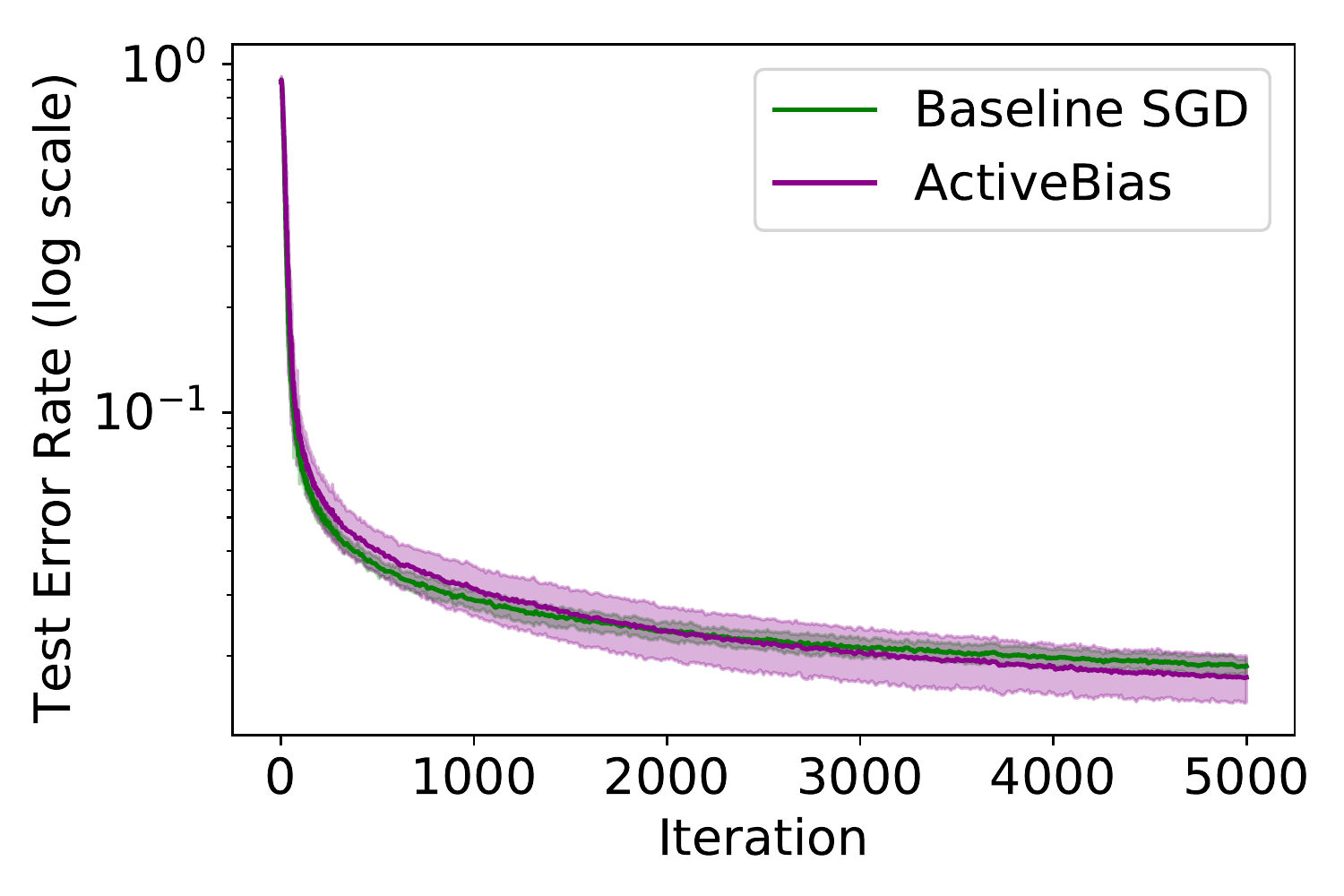}
\label{fig:MNIST:ActiveBIas}
}
\subfigure[VannilaPDS comparison]{
\includegraphics[width = 0.45\textwidth]{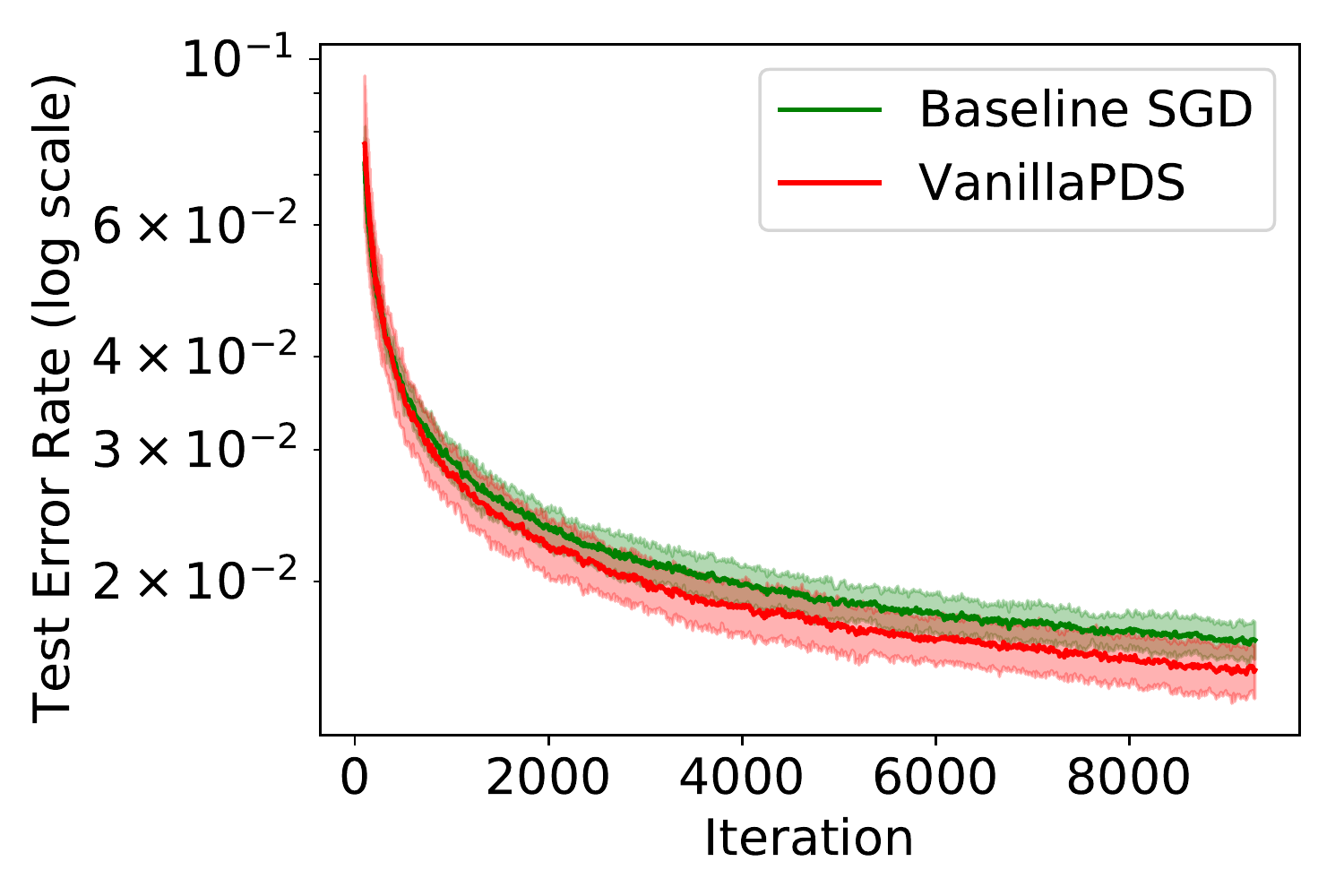}
\label{fig:MNIST:VannilaPDS}
}
\subfigure[EasyPDS comparison]{
\includegraphics[width = 0.45\textwidth]{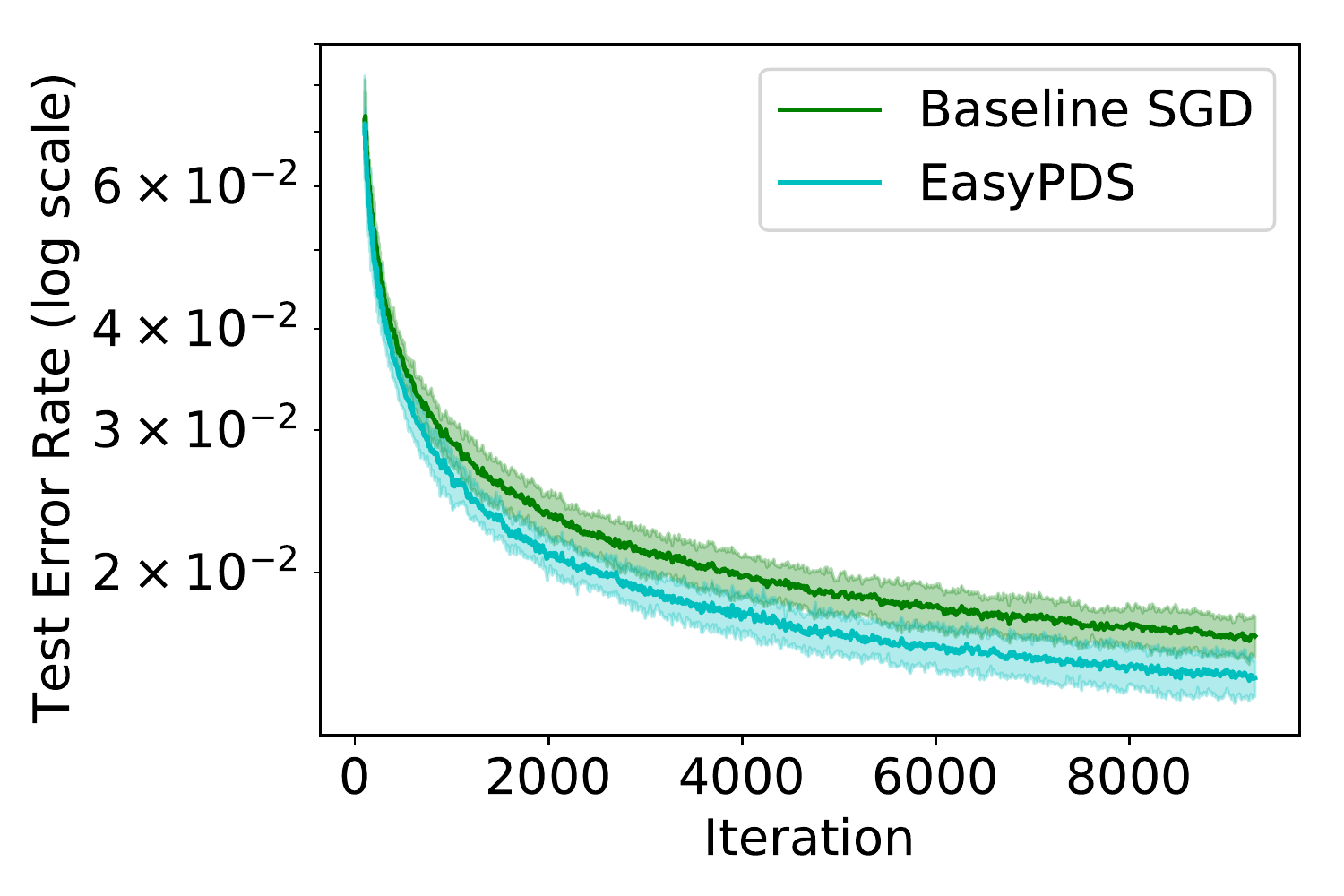}
\label{fig:MNIST:EasyPDS}
}
\subfigure[DensePDS comparison]{
\includegraphics[width = 0.45\textwidth]{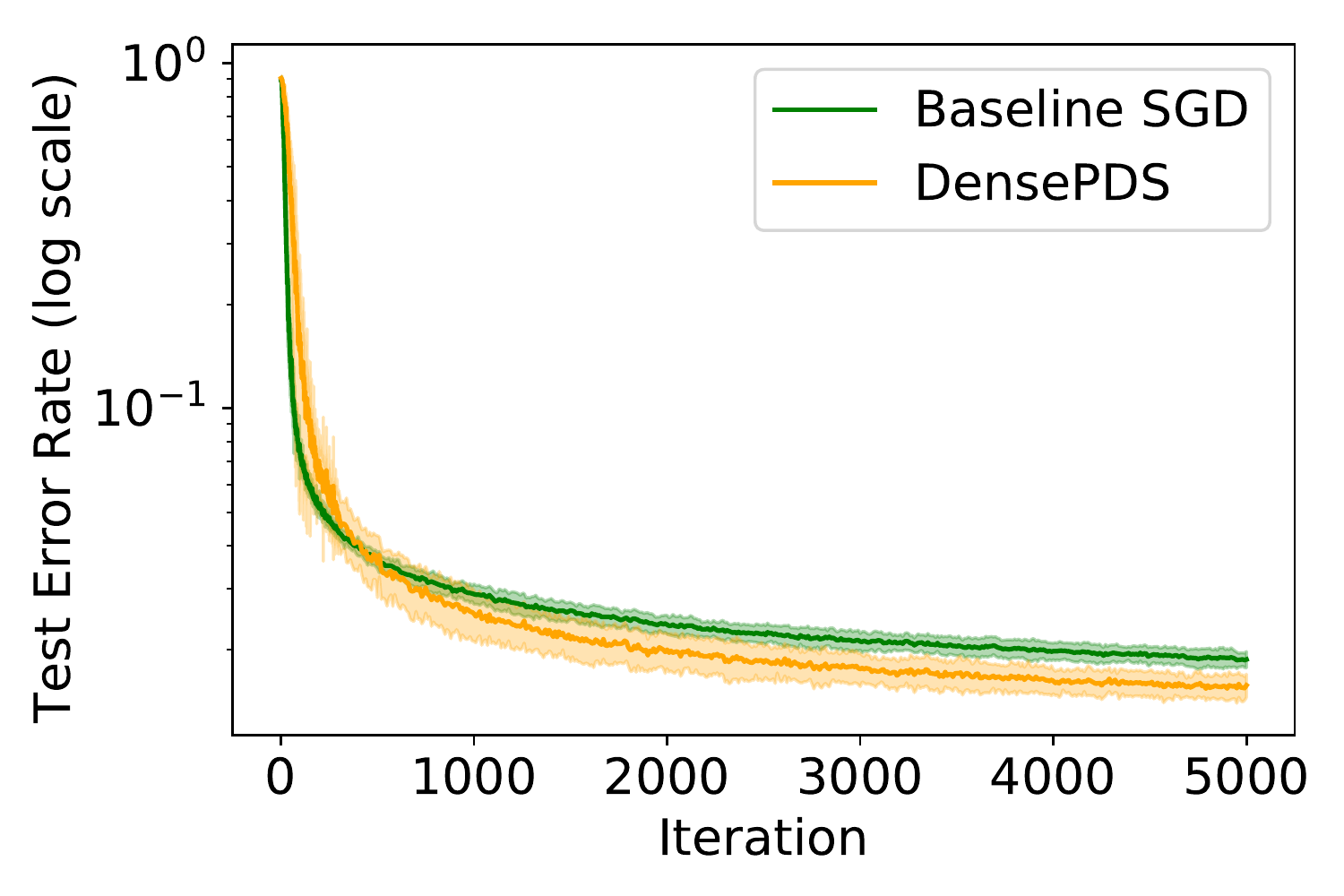}
\label{fig:MNIST:DensePDS}
}
\subfigure[AnnealPDS comparison]{
\includegraphics[width = 0.45\textwidth]{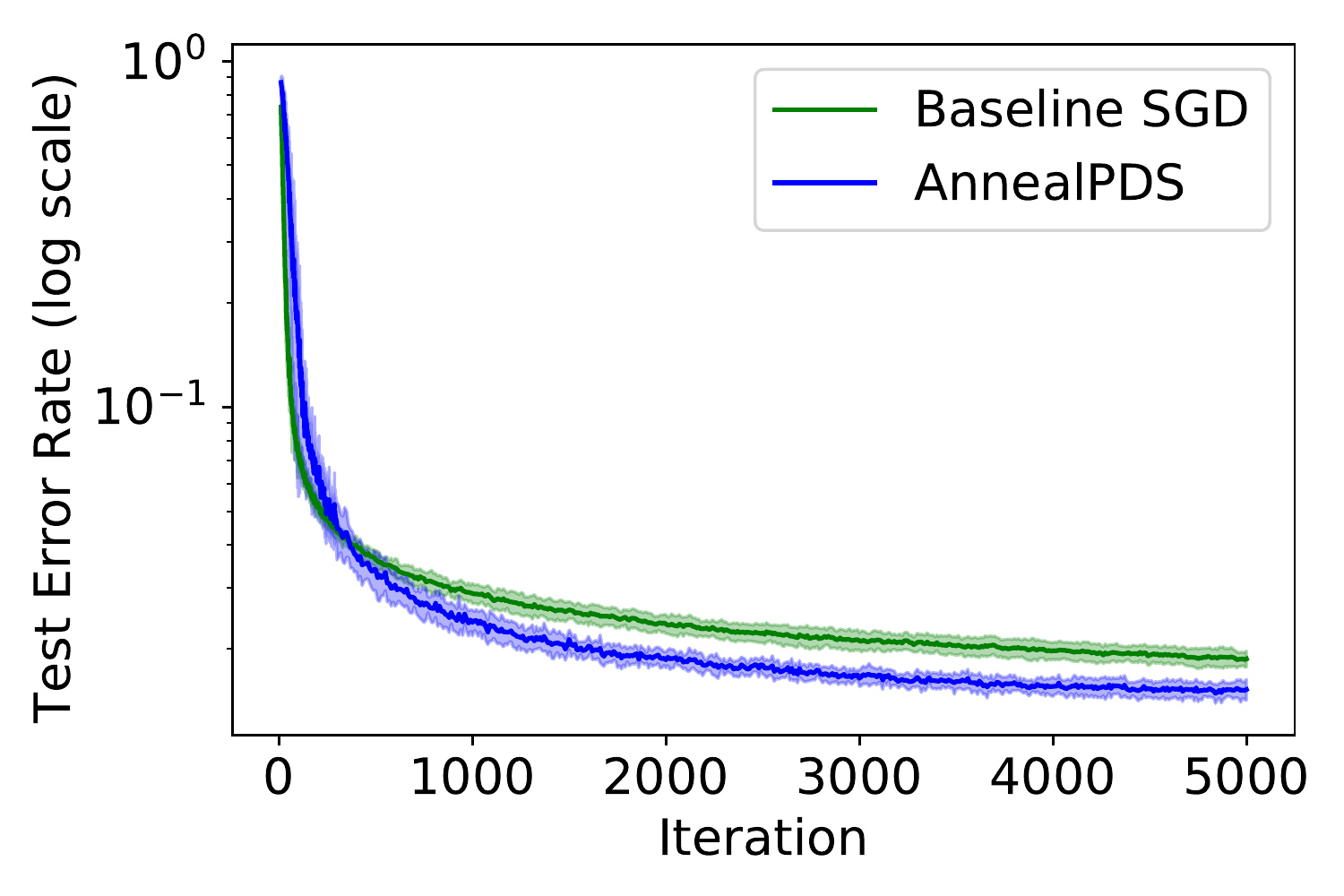}
\label{fig:MNIST:AnnealPDS}
}
\caption{MNIST experiments performance comparison. Panel (a) shows the mean performance of multiple variants of our methods and two baseline methods. Panel (b) shows the ActiveBias baseline performance in comparison with traditional SGD performance (main baseline). Panel (c) - (f) show the performance of our method compared to the traditional SGD. All the experiments are run ten times. The mean and standard deviations are visualized.}
\end{figure}

Figure \ref{fig:Annel} shows the annealing parameter that is used in the MNIST experiment in Section \ref{sec:MNIST}. The original  distribution of data with different mingling indices is $h=[0.9017, 0.0474, 0.0212, 0.013, 0.0096, 0.0071]$. We uses an annealing schedule where  $\pi_n$ is set to normalized $h^{1/{log(0.01n+1)}}$. In this way, data with mingling index 0 (easy examples) are sampled more often in the early iterations during training.   
\vspace{20pt}

\begin{figure}[t]
\floatbox[{\capbeside\thisfloatsetup{capbesideposition={right,top},capbesidewidth=8cm}}]{figure}[\FBwidth]
{\caption{The annealing schedule that is used in the 'Anneal PDS' MNIST experiment in the paper. The density is deceasing for easy examples (with mingling indices $0$), while increasing for non-trivial examples (with bigger mingling indices). }\label{fig:Annel}}
{\vspace{-10pt}\includegraphics[width=6cm]{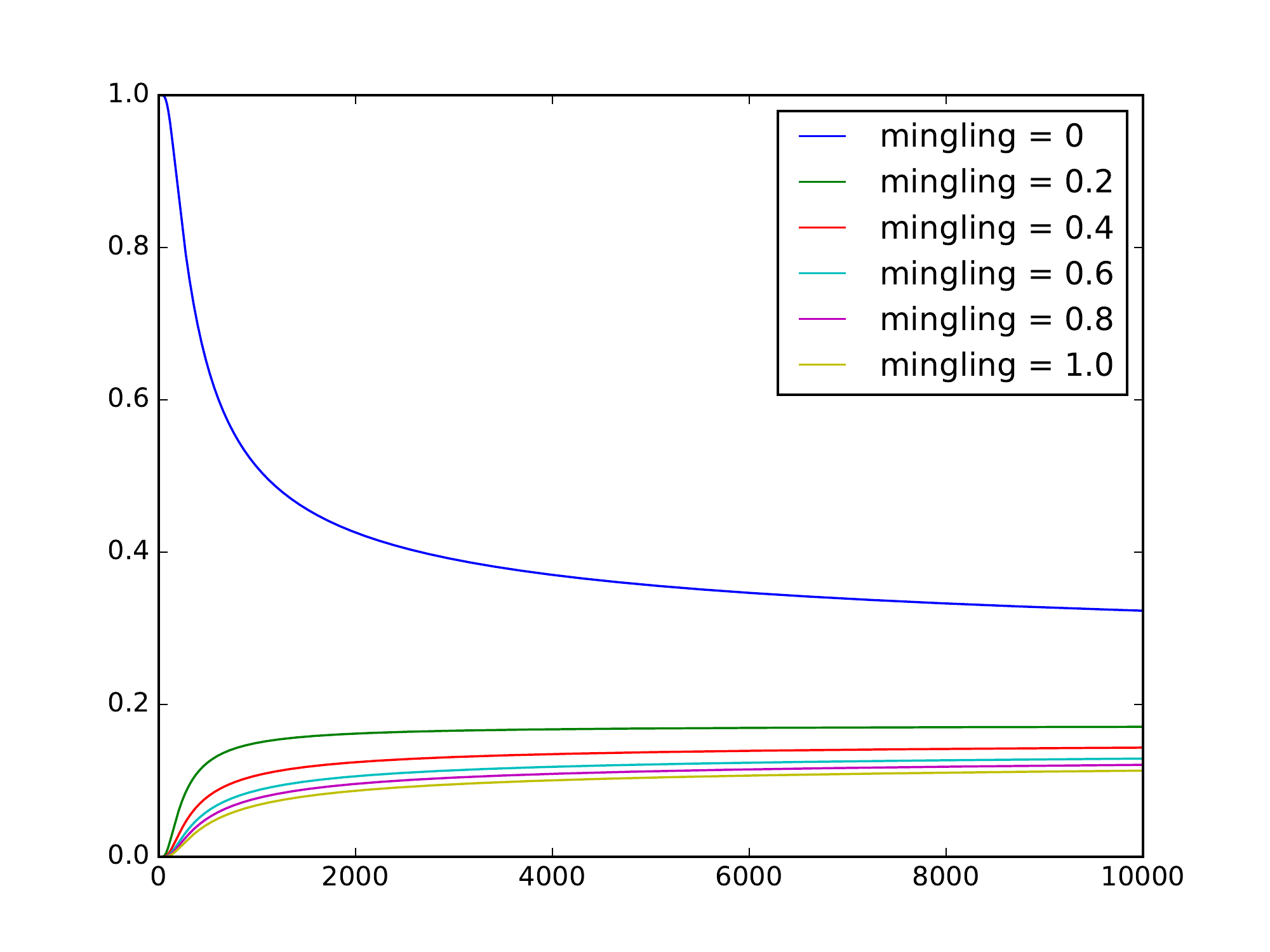}}
\end{figure}

Figure \ref{fig:Speech_other_hist_mingling} shows the distributions of data with different mingling index over these twelve classes in the speech experiment. We can see that most points with mingling indices  $0$ are from the ``silence'' class where the data are artificially constructed from background noise. Data from other classes, in general, have higher mingling indices, since they are not well separated as shown in the t-SNE plot in the paper. Additionally, we can see in Figure \ref{fig:Speech_other_hist_mingling} (b) that most data points with mingling index $1$ are from the ``unknown'' class, that is the noisiest class because it corresponds to a number of different words.

\begin{figure*}[t]
\centering
\subfigure[Data Visualization]{
\includegraphics[width = 0.43\textwidth, height = 5cm]{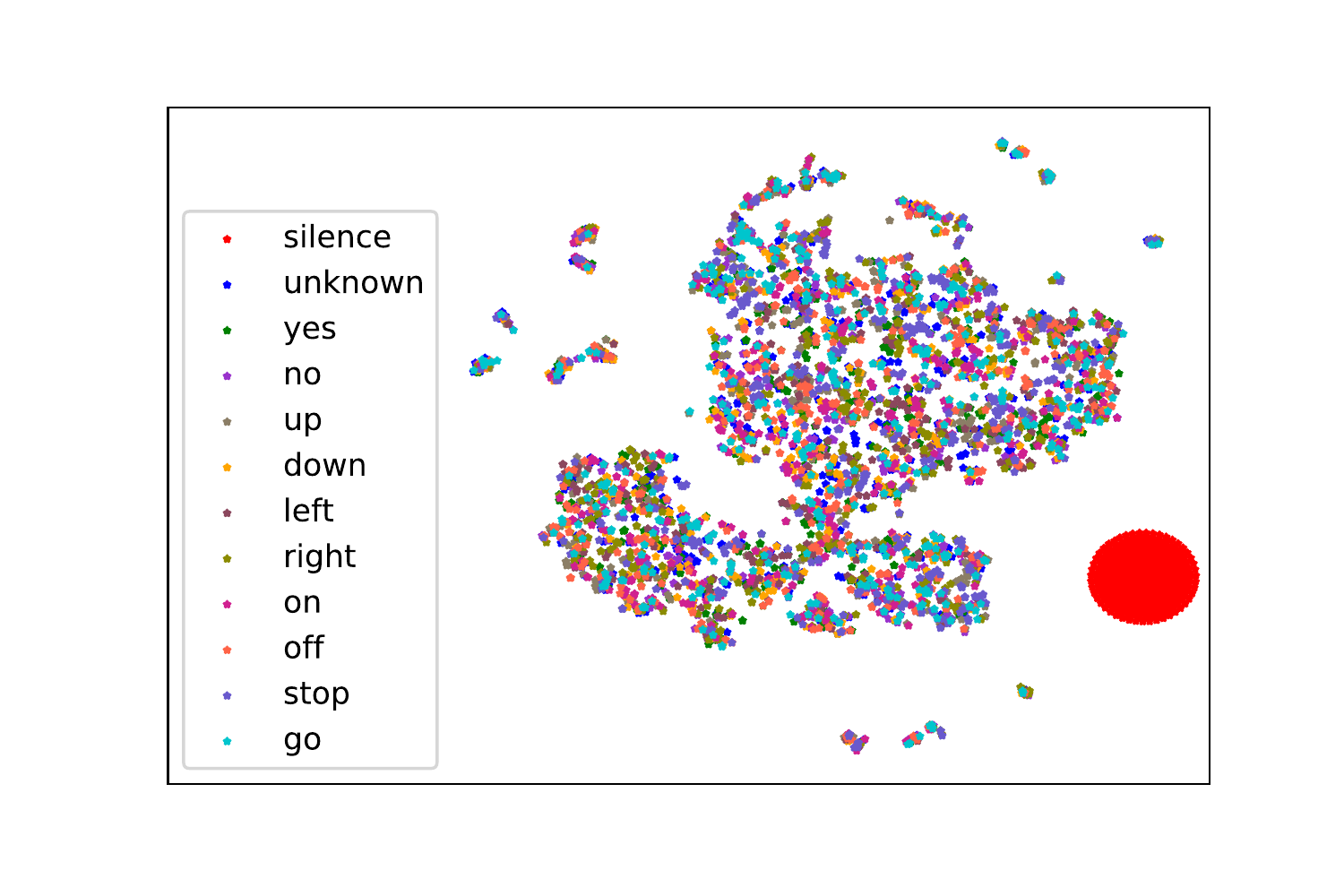}
\label{fig:Speech:tSNE}
}
\hspace{-5pt}
\subfigure[Mingling index distribution]{
\includegraphics[width = 0.45\textwidth, height = 5cm]{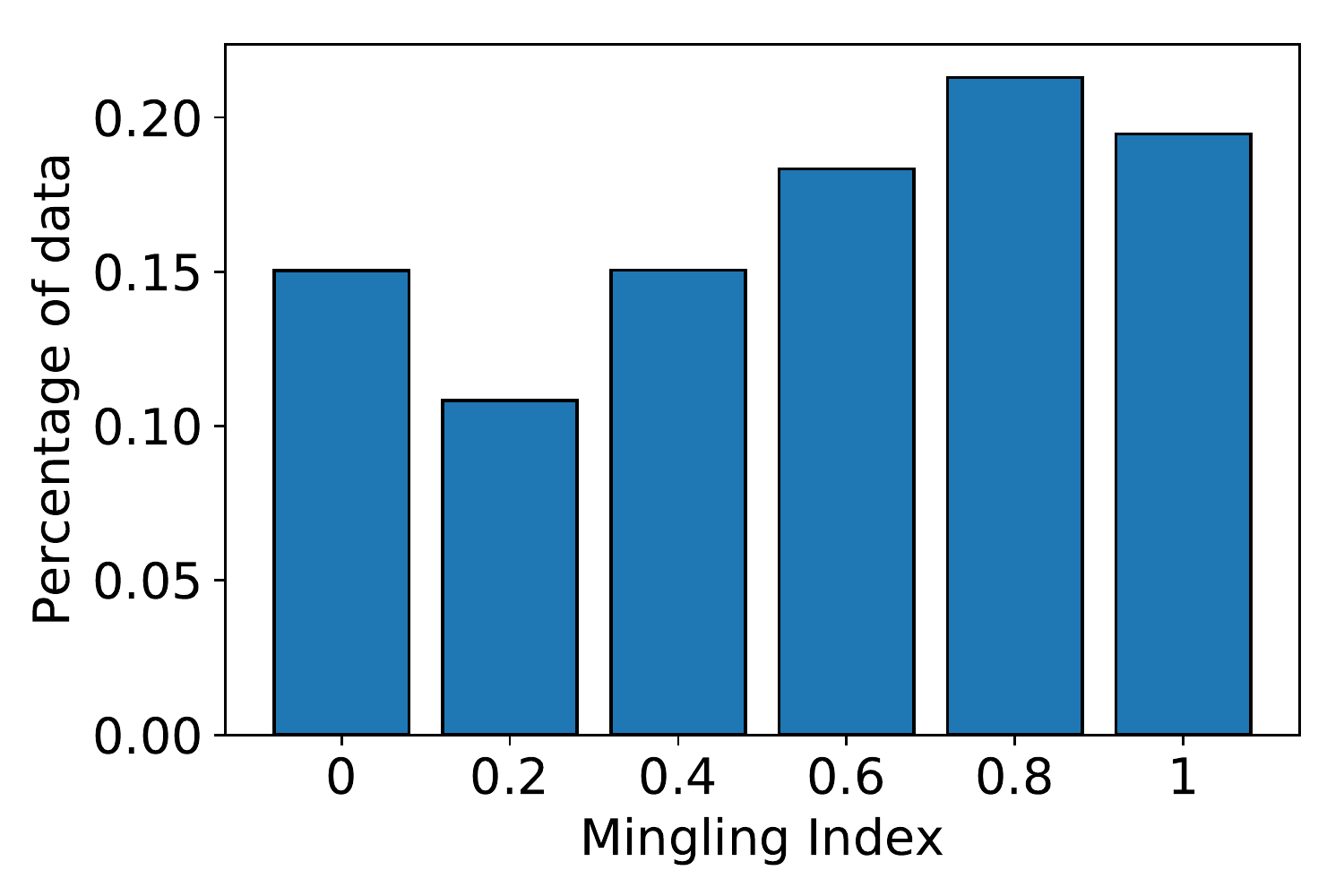}
\label{fig:Speech:mingling}
}
% \hspace{-10pt}
% \subfigure[Performance Comparison]{
% \includegraphics[width = 0.33\textwidth, height = 3.5cm]{TFSAR/CompareFourMethod_RB1000_06_bigFont_PDS.pdf}
% \label{fig:Speech:performance}
% }
\caption{\footnotesize Speech experiment. Panel (a) visualizes the dataset using t-SNE.
Panel (b) shows the histogram of data with mingling index $0$. Panel }
\label{fig:Speech}
\end{figure*}

\begin{figure*}[t]
\centering
\subfigure[Mingling index 0]{
\includegraphics[width = 0.31\textwidth]{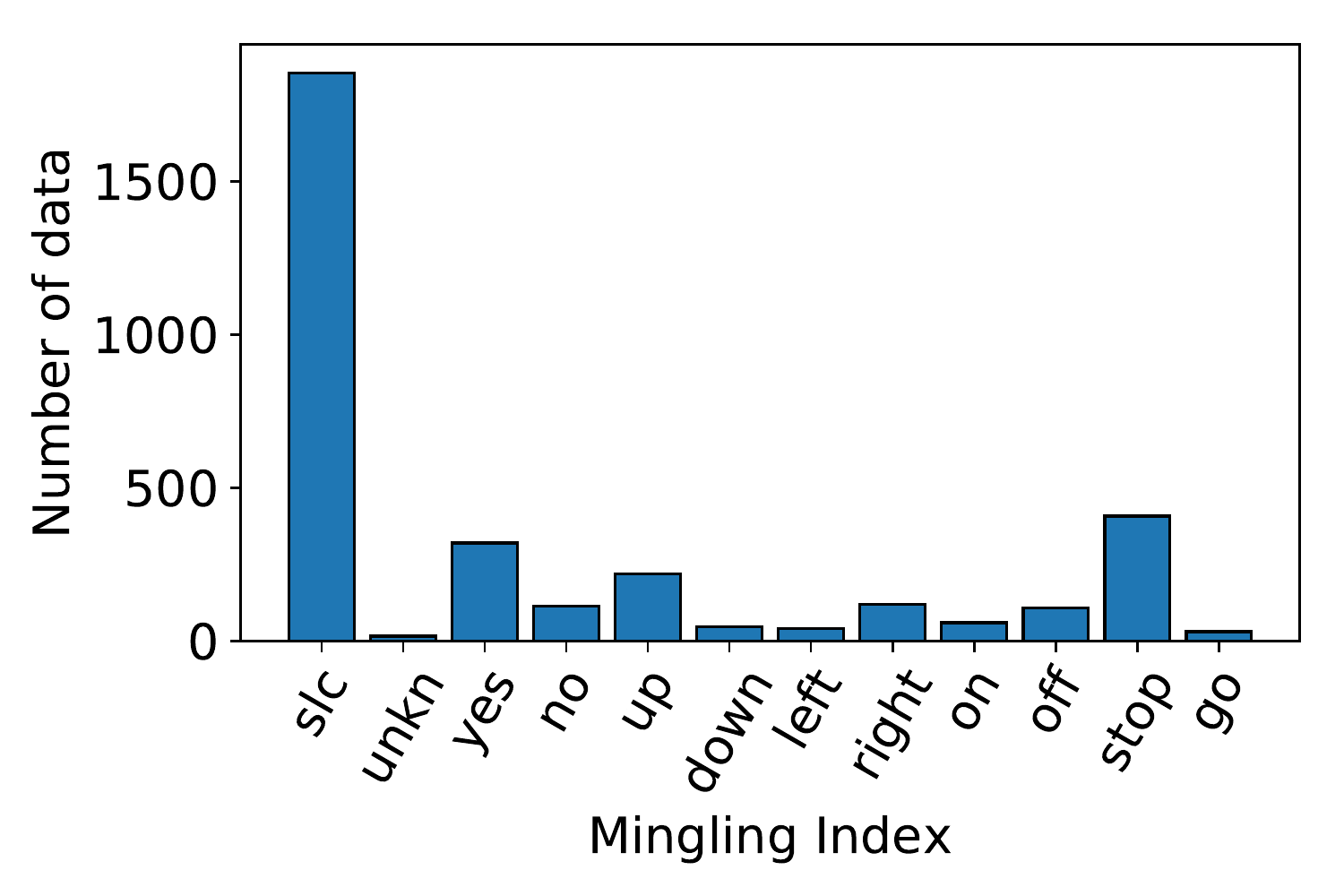}
}
\subfigure[Mingling index 0.2]{
\includegraphics[width = 0.31\textwidth]{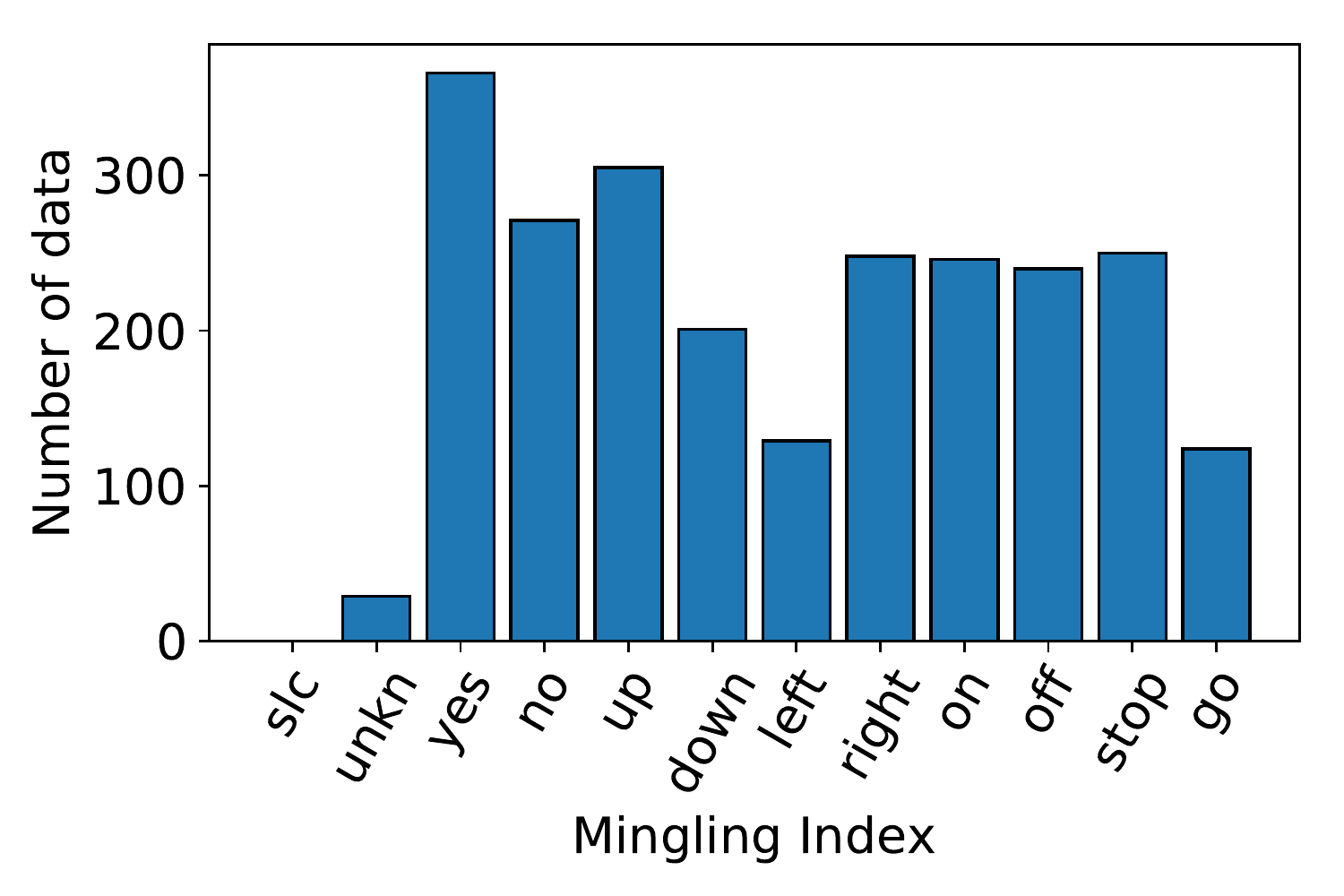}
}
\subfigure[Mingling index 0.4]{
\includegraphics[width = 0.31\textwidth]{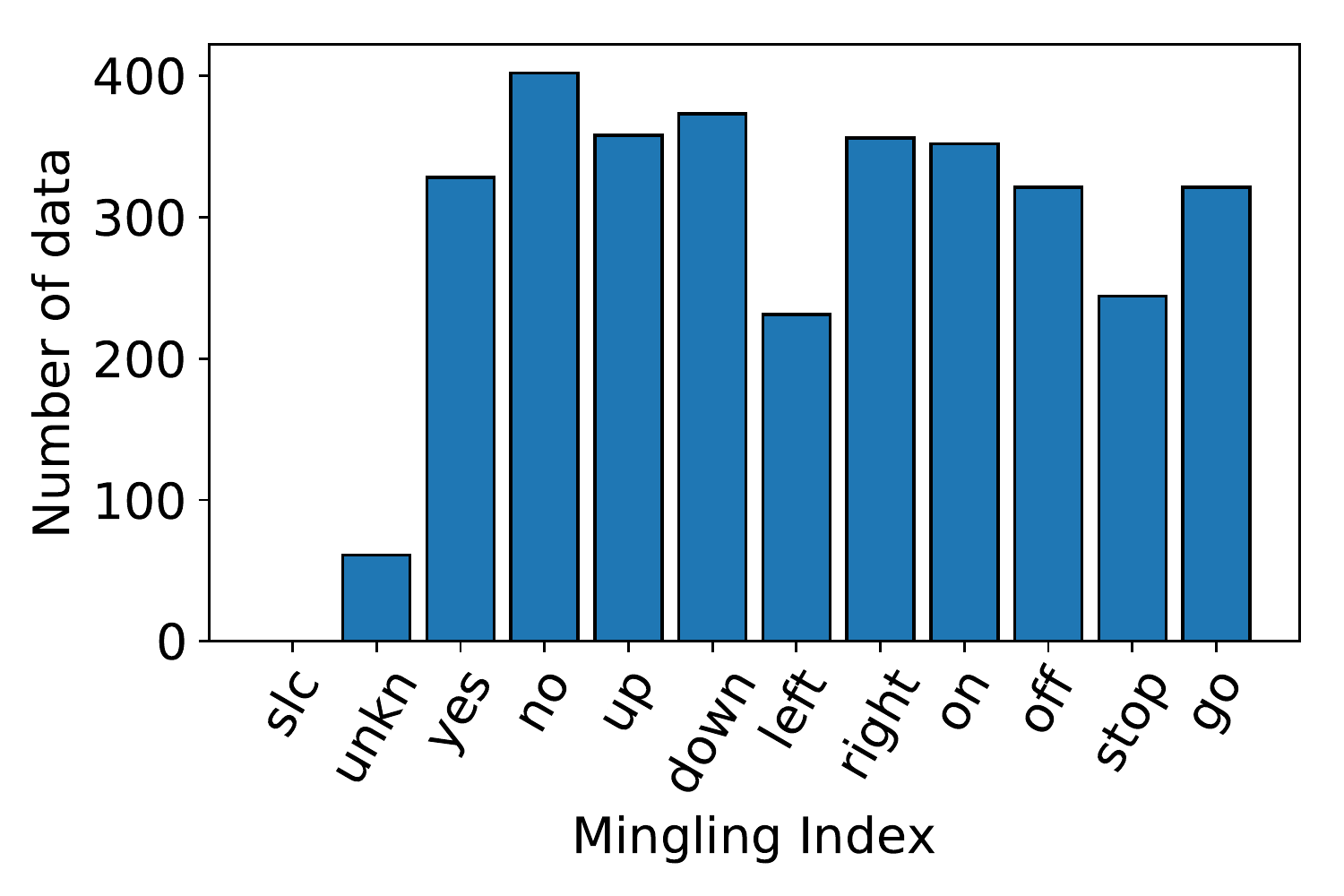}
}

\subfigure[Mingling index 0.6]{
\includegraphics[width = 0.31\textwidth]{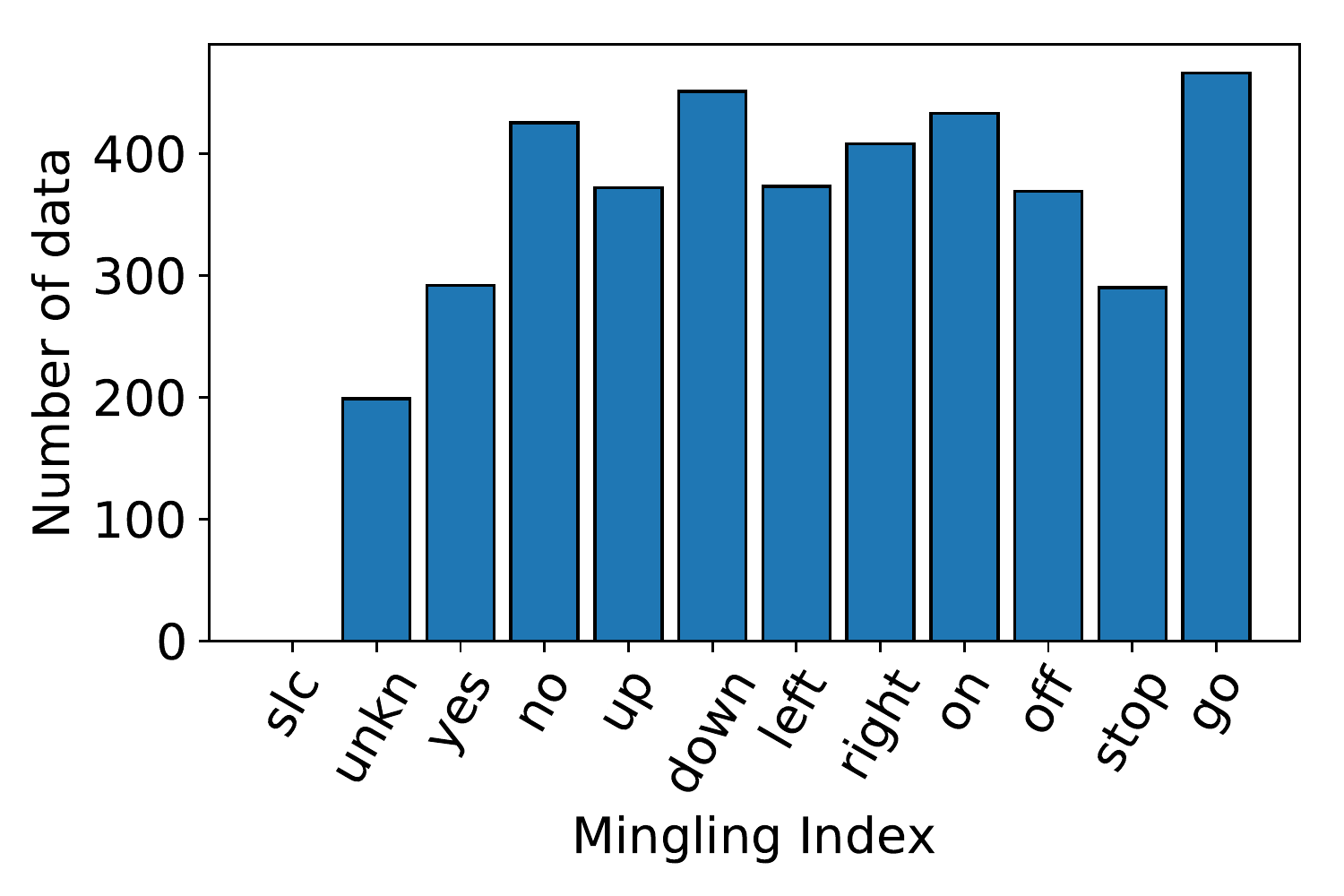}
}
\subfigure[Mingling index 0.8]{
\includegraphics[width = 0.31\textwidth]{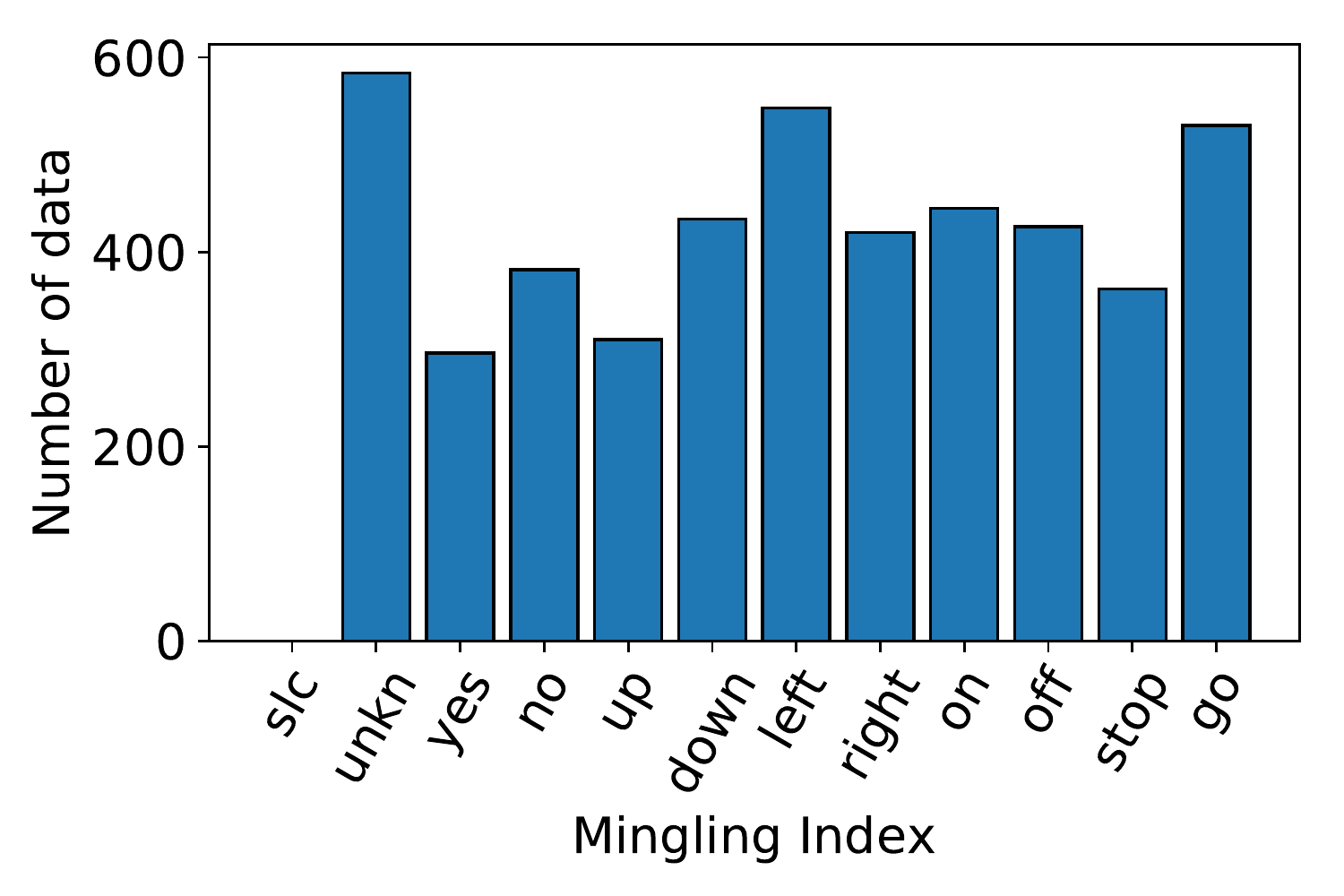}
}
\subfigure[Mingling index 1.0]{
\includegraphics[width = 0.31\textwidth]{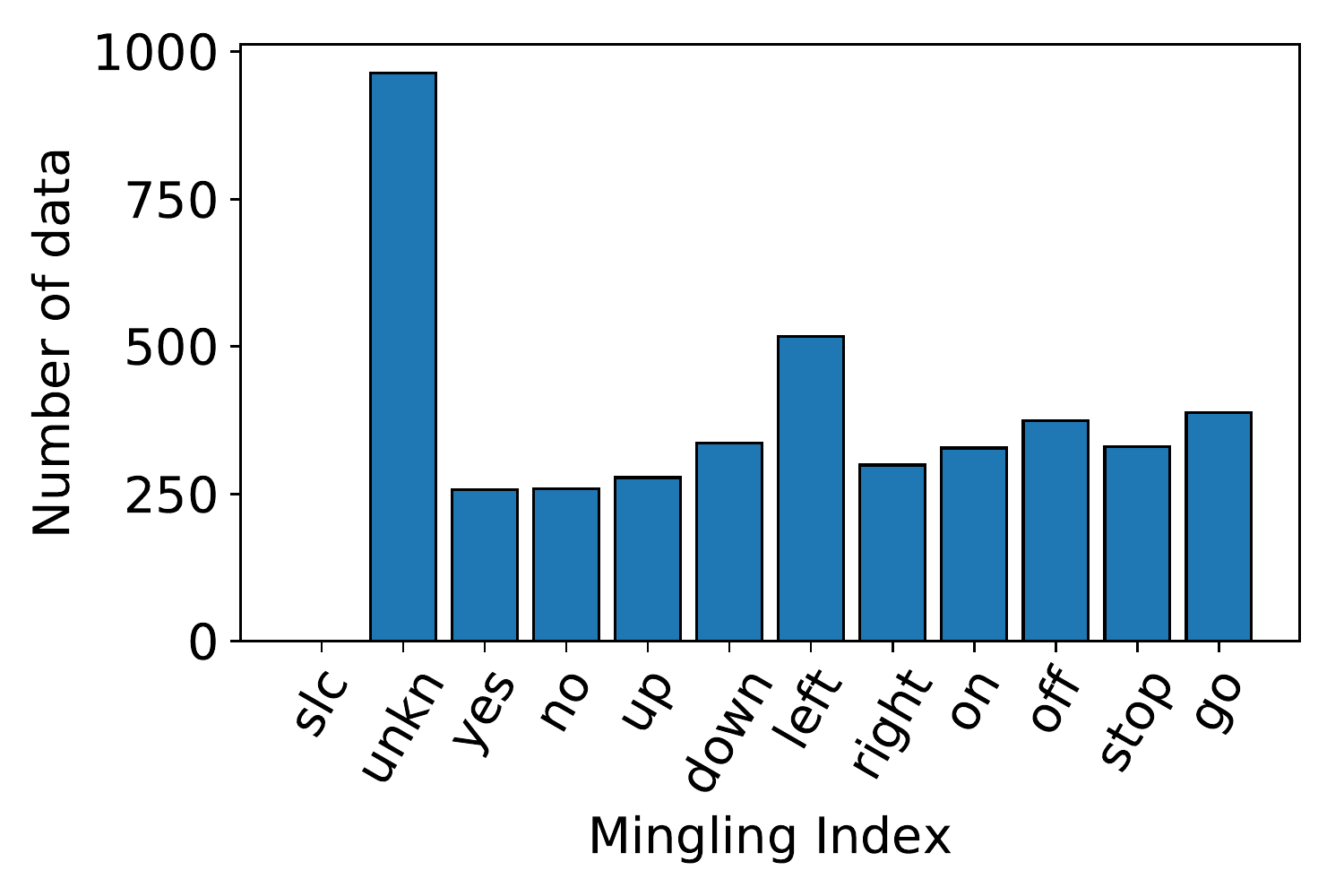}
}
\caption{Speech experiment. The distribution of data with respect to different class labels under different mingling index values. 'slc' indicates the silence class, and 'unkn' indicates the unknown class where the utterance does not belong to any of these ten target command words.  }
\label{fig:Speech_other_hist_mingling}
\end{figure*}
\end{document}